\documentclass{article}

\usepackage{nips_2018}

\usepackage{bbm}
\usepackage[utf8]{inputenc} 
\usepackage[T1]{fontenc}    
\usepackage{hyperref}       
\usepackage{url}            
\usepackage{booktabs}       
\usepackage{amsfonts}       
\usepackage{nicefrac}       
\usepackage{microtype}      
\usepackage[utf8]{inputenc} 
\usepackage[T1]{fontenc}    
\usepackage{hyperref}       
\usepackage{url}            
\usepackage{booktabs}       
\usepackage{amsfonts}       
\usepackage{nicefrac}       
\usepackage{microtype}      
\usepackage{todonotes}
\usepackage{amsmath}
\usepackage{amsthm}
\newtheorem{theorem}{Theorem}

\newtheorem{lemma}[theorem]{Lemma}
\newtheorem{corollary}{Corollary}[theorem]
\usepackage{media9}

\usepackage{graphicx}
\usepackage{caption}
\usepackage{subcaption}
\usepackage{algorithm}
\usepackage{algorithmic}
\usepackage{color}
\allowdisplaybreaks

\def\p{p}
\def\q{q}
\def\Z{Z}

\def\KL{\text{KL}}

\def\KL{\mathsf{KL}}
\def\EE{\mathbb{E}}
\def\RR{\mathbb{R}}

\def\Prob{\mathbb{P}}

\title{On Exploration, Exploitation and Learning in Adaptive Importance Sampling}

\usepackage{titlesec}
\titlespacing\section{0pt}{6pt plus 2pt minus 2pt}{2pt plus 2pt minus 0pt}
\titlespacing\subsection{0pt}{6pt plus 2pt minus 2pt}{2pt plus 2pt minus 0pt}
\titlespacing\subsubsection{0pt}{6pt plus 2pt minus 2pt}{2pt plus 2pt minus 0pt}

\begin{document}

\maketitle
\begin{abstract}

We study adaptive importance sampling (AIS) as an online learning problem and argue for the importance of the trade-off between exploration and exploitation in this adaptation. Borrowing ideas from the bandits literature, we propose Daisee, a partition-based AIS algorithm. We further introduce a notion of regret for AIS and show that Daisee has
$\mathcal{O}(\sqrt{T}(\log T)^{\frac{3}{4}})$ cumulative pseudo-regret, where $T$ is the number of iterations.
We then extend Daisee to adaptively  learn a hierarchical partitioning of the sample space for more efficient sampling and confirm the performance of both algorithms empirically.

\end{abstract}

\section{Introduction}

Monte Carlo (MC) methods form the bedrock upon which significant sections of probabilistic machine learning and computational statistics rest. An important MC technique which forms the basis for many others is importance sampling (IS). Let $\pi(x) = f(x)/Z$ be a target density which can be evaluated pointwise up to an unknown normalising constant $Z$ and let $q(x)$ be a proposal distribution from which samples can be drawn and can be evaluated pointwise. 
IS works by drawing a sequence of samples from $q(x)$ and using these to 
estimate $Z$ and
provide an empirical measure $\hat{p}(\cdot)$ that we can use at later time to calculate target statistics $\EE_\pi[\phi(x)]$ for any arbitrary test function $\phi(x)$. Let $w(x_t) = f(x_t)/q(x_t)$ be the importance weight of $x_t$.  Then,
\begin{align}
Z = \int &f(x)dx = \EE_q[w(X)] \approx \frac{1}{T}\sum_{t=1}^T w(x_t), \\
\hat{p}(\cdot) &= \frac{\sum_{t=1}^T w(x_t) \delta_{x_{t}}(\cdot)}{\sum_{t=1}^T w(x_t)},
\\
\EE_\pi[\phi(x)] 
&\approx 
\int \hat{p}(x) \phi (x) dx = 
\frac{\sum_{t=1}^T w(x_t) \phi(x_t)}{\sum_{t=1}^T w(x_t)}. \label{eq:ISestimates}
\end{align}
Note that the estimate for $Z$ is unbiased, but that for the target statistics is biased but consistent, provided $\text{support}(f) \subset \text{support}(q)$ and $\EE_q[w(X)\phi(X)] < \infty$ \citep{mcbook}. 

The efficiency of IS for a general $\phi(x)$ is governed by the choice of proposal $q(x)$, with the intuition that the closer $q$ is to $\pi$ the better.
Adaptive Importance Sampling (AIS) techniques \citep{liu2008monte,cappe2004population,cornuet2012adaptive,cappe2008adaptive} attempt to improve the efficiency of IS by adapting the proposal to be closer to the target, producing a sequence of proposals $q_1,q_2,\ldots$.  The IS estimates \eqref{eq:ISestimates} still apply with $q(x_t)$ replaced by $q_t(x_t)$. 
The basic idea is that previous samples provide information about the target distribution and this information can be used to improve the proposal by altering it to better match the target~\citep{mcbook,bugallo2017adaptive}. 

Viewed in this way, AIS is, in effect, an online learning problem: that of learning the target density $\pi$ by iteratively querying it.  As opposed to the typical setup of density estimation, where each query is an iid sample from the target, here 
$q_t$ is both our current estimate of the target, as well as our tool for querying the target.  
This exposes a trade-off between exploration and exploitation.  We would like our proposal $q_t$ to be as close as possible to the target, so that our IS estimate is as good as possible (exploit). At the same time, $q_t$ directs where queries of the target are made and probability mass needs to be spread over the sample space where we have high uncertainty of the target, so that we may query and reduce our uncertainty to improve our estimate in the future (explore). 

In this paper we take the first steps towards developing AIS methods which optimally trade off exploration and exploitation by bringing to bear ideas from the rich literature on online learning, in particular that on multi-armed bandit algorithms \citep{berry1985bandit,auer2002finite,agrawal2012analysis}. We propose the \emph{aDAptive Important Sampling via Exploration and Exploitation} (Daisee), an algorithm which operates by partitioning the sample space into $K$ disjoint subsets, and adapting the proposal probability of each subset by combining an exploitative estimation term and an explorative optimism boost. 
Based on the analysis of~\citep{chatterjee2015sample}, which showed that the Kullback-Leibler (KL) divergence from the target to the proposal is the appropriate measure of the (in)efficiency of inference using IS,
we propose using the KL divergence as the measure of loss for an AIS scheme, and show that, 
under mild assumptions,
Daisee achieves a cumulative pseudo-regret of $O(\sqrt{T}(\log T)^\frac{3}{4})$ where $T$ is the  number of iterations.
This means that its per-iteration pseudo-regret asymptotes to zero.
We also show a similar result when generalising to $\alpha$-divergence \citep{cichocki2010families} based regrets, with $\alpha=1$ reducing to the KL case, while $\alpha=2$ corresponds to the variance of the importance weights $\frac{\pi(x)}{q(x)}$.
Finally, we introduce and empirically investigate HiDaisee, a hierarchical extension to Daisee in which the partitions themselves are also simultaneously learned in an online fashion, leading to a more efficient overall proposal.  


\section{Related Work}\label{sec:literature}

The general topic of adaptive MC has  received extensive attention in the literature (see e.g.~\citep{bugallo2017adaptive,mcbook} and  references therein).  There are a wide range of problems such as rare event simulation~\citep{owen2000safe,owen2017importance} for which IS is more useful. Compared with MCMC, IS produces a marginal likelihood estimate. It is also the key component of many advanced MC methods.

Of particular relevance to our work, 
\citep{friedman1991multivariate,lepage1978new} tackle the problem of multidimensional integration by recursive rectangular partitioning of the target space, where regions with higher contribution to the integral are subdivided into more subregions. Here we instead focus on the learning of the proposal densities with exploration-exploitation trade-off in mind rather than integration. Meanwhile, \citep{he2014optimal} bound the regret for multiple importance sampling where the proposal distribution is a mixture of a finite number of proposal distributions, by bounding the variance of the importance sampling estimates with the help of control variates. \cite{cappe2008adaptive} use an entropy criterion instead to learn the weights and component parameters of a mixture importance sampling density.

On the other hand, methods which address the exploration-exploitation trade off have been well-studied in online learning, most successfully under the banner of bandit algorithms \citep{berry1985bandit,bubeck2012regret,srinivas2009gaussian}, with Upper Confidence Bound (UCB) \citep{auer2002finite} and Thompson sampling \citep{agrawal2012analysis} being popular approaches. In the standard multi-armed bandit problem, one has a finite number arms and must choose an arm to pull at each iteration, returning a random reward which is distributed differently for each arm.
The aim is to maximise rewards in the long run by finding the arm with the highest average reward.  UCB methods operate by maintaining an estimate of the expected reward for each arm and picking arms according to the estimates plus optimism boosts which are larger for arms where our estimates are less certain to encourage exploration.  \cite{auer2002finite} show that UCB optimally trades-off exploration and exploitation by showing that the cumulative regret (relative to an oracle which knows the optimal arm) grows logarithmically in the number of iterations, the best growth rate achievable \citep{lai1985asymptotically}. \cite{bubeck2013bandits, sen2018contextual} extend the bandit algorithm to fat tailed distributions in absence of sub-Gaussianity.

The concept of trading off exploration versus exploitation has not been widely considered in the adaptive MC setting, though there are a number of recent works. \cite{neufeld2014adaptive} consider the problem of choosing between a number of MC estimators, where the goal is to minimise the mean squared error by
finding the estimator with lowest variance.  This is still inherently a best-arm optimisation problem, rather than a true proposal adaptation. \cite{carpentier2011finite} and \cite{lepretre2017multi} consider the problem of stratified sampling for MC integration, where each stratum is viewed as an arm and the mean in each stratum is estimated. They show that in this setting the optimal approach is to choose arms in proportion to the standard deviation of the estimate produced by a single draw from the stratum. Our setting (and resultant approach) is distinct to \citep{carpentier2011finite,lepretre2017multi} as we aim to approximate a distribution, rather than performing an integration. 

\citep{rainforth2018inference} also realised the need for adaptive MC methods to trade-off between
exploration and exploitation and they also take steps towards designing a framework for adapting the hierarchical
partition of the space.  However, their
focus is on the design of a new class of meta-inference algorithms that fall outside the AIS framework and they do not
consider any notion of regret nor provide a regret analysis.  Our focus here is more on the exploration-exploitation trade-off in AIS, establishing a formal notion of regret so
that this can be viewed as an online learning problem, and establishing a theoretical regret bound for our procedure. 

\section{Daisee}

\label{sec:AdaIS}

\begin{algorithm}[t]
\begin{algorithmic}[1]
\STATE \textbf{Input:} partitioning $\mathcal{X}_1, \dots, \mathcal{X}_K$, partition proposals
$g_1(x),\dots,g_K(x)$, unnormalized target distribution $f(x)$
\STATE Draw one sample from each subset to initialise estimates, denoting these as $x_{1},\dots,x_{K}$
\FOR {$t=K+1$ \TO $T$}
\STATE {Compute proposal probabilities $\{q_{at}\}_{a=1}^K$ using \eqref{eq:proposal}, with estimates $\{\hat{Z}_{a,t-1}\}_{a=1}^K$  given by \eqref{eq:Zat}}
\STATE {Draw an arm $A_t \sim \text{Discrete}(q_{1t},\ldots,q_{Kt})$} \label{line:A_t}
\STATE {Draw a sample $x_t \sim g_{A_t}(x)$} \label{line:x_t}
\STATE {Compute localized weight $Y_{A_tt} = f(x_t)/g_{A_t}(x_t)$ and update estimate for $Z_a$ using~\eqref{eq:Zat}}
\ENDFOR
\end{algorithmic}
\caption{Daisee (aDaptive Importance Sampling with Exploration-Exploitation)}\label{daisee}
\end{algorithm}

We now propose Daisee, an AIS method which has a similar flavour to UCB. Assume we have a partition of our sample space $\mathcal{X}$ into $K$ disjoint subsets $\{\mathcal{X}_a\}_{a=1}^K$, where each subspace $\mathcal{X}_a$ can be thought as one bandit arm, and we have a fixed tractable proposal distribution $g_a(x)$ restricted on each subset $\mathcal{X}_a$ for $a\in\{1,\ldots,K\}$. For example, if  $\mathcal{X}$ is compact we can choose the uniform distribution $g_a(x) = \mathbb{I}(x \in \mathcal{X}_a)/\int_{\mathcal{X}_a} dx$. 
Daisee uses proposal distributions of  form
\begin{align}
q_t(x) = \sum_{a=1}^K q_{at} g_a(x)\mathbb{I}(x \in \mathcal{X}_a)
\label{eq:proposalform}
\end{align}
where $\sum_a q_{at}=1$ and the probability masses $q_{at}$ of the subsets are to be adapted in the scheme. 
Note that whereas typical bandit algorithms only aim to establish the best arm, we are instead looking to asymptotically
learn a distribution over how often each arm should be pulled.
At each iteration $t$, we query the target by drawing a sample $x_t$ from $q_t$. This can be achieved by first sampling an arm 
$A_t \sim \text{Discrete}(q_{1t},\ldots,q_{Kt})$,
then drawing a sample $x_t\in \mathcal{X}_{A_t}$ from the subset proposal $g_{A_t}(x)$, as per lines~\ref{line:A_t}
and~\ref{line:x_t} respectively in Algorithm~\ref{daisee}.

Let the ratio $Y_{at} := \frac{f(x_t)}{g_{a}(x_t)}$ denote the ``localized'' importance weight for a sample
in subset $a$, noting that each $Y_{at}$

is an unbiased estimator of the (unnormalised) target probability $Z_a$ of  $\mathcal{X}_a$:
\begin{align}
&\pi_a := \int_{\mathcal{X}_a} \pi(x)dx = \frac{Z_a}{Z} = \frac{Z_a}{\sum_{b=1}^K Z_b}, \\
&Z_a := \int_{\mathcal{X}_a} f(x)dx = \EE_{g_a}\left[\frac{f(x)}{g_a(x)}\right].
\label{eq:Z_a}
\end{align}
At each iteration this leads to the following MC estimate for each $Z_a$
\begin{align}
\hat{Z}_{at} := \frac{1}{N_{at}}\sum_{s\le t:A_s=a}Y_{as}
  \label{eq:Zat}
\end{align}
where $N_{at} = \#\{s:s\le t,A_s=a\}$ is the number of times subset $a$ was chosen up to time $t$.
We note an unusual behaviour of the estimates $\hat{Z}_{at}$: although each $Y_{at}$ is unbiased, each $\hat{Z}_{at}$ is biased
because the $N_{at}$'s are random variables correlated with previous samples.
We sidestep this by upper bounding the probability that $\hat{Z}_{at}$ deviates significantly, simultaneously for all $t$. 

Naively, we can use the estimates $\hat{Z}_{at}$ to construct the next proposal, via $q_{a,t+1}\propto \hat{Z}_{at}$.
The problem is that if by chance   $\hat{Z}_{at}$ for some subset $a$ is too small, the resulting underestimated proposal probability will result in low probability for the subset to be picked in future, and hence the bad estimate may not be corrected. This is a symptom of under-exploration. As in UCB, we therefore encourage exploration using an optimism boost $\sigma_{at}$:
\begin{align}
    q_{a,t+1} = \frac{ \hat{Z}_{at}  + \sigma_{at}}{\sum_{b=1}^K( \hat{Z}_{bt}  + \sigma_{bt})}
    \label{eq:proposal}
\end{align}
where $\sigma_{at}$ should be decreasing with $N_{at}$ but grow with $t$. The intuition, which will be formalized in the next section, is that if we have not explored the subset $a$ sufficiently,  $\sigma_{at}$ is relatively large, which compensates and boosts $q_{at}$, allowing us to have higher chance to explore subset $a$ and correct the under-estimate. The growth with $t$ is to ensure sufficient exploration of all subsets over time.

As shown in Algorithm~\ref{daisee}, Daisee iterates between drawing samples from $q_t$ defined by ~\eqref{eq:proposalform}, and using this sample to update our estimates \eqref{eq:Zat}, \eqref{eq:proposal}. 
The next section shows that this  simple
approach leads to a low-regret AIS strategy, in the sense that its per-round pseudo-regret tends to zero.

\section{Regret Formalisation and Analysis}
\label{sec:theorem}

Having introduced Daisee algorithmically, we now switch focus to formalising a notion of regret suitable for the AIS
setting and analysing this regret  for Daisee.

In the nomenclature of bandits, $Z_a$ can be thought of as the expected reward of arm $a$, while $Y_{at}$ is the random reward received at iteration $t$. Note however that our aim is not to maximise expected reward, but rather to maximise the efficiency of the resulting IS sampler. Thus a measure of the (in)efficiency of the proposals $q_t$ is required, so that we can define what makes an optimal proposal among the class in (\ref{eq:proposalform}), and the regret of using a proposal relative to the optimal. We will use the KL divergence $\KL(\pi\|q_t)$ as this has been shown by \citep{chatterjee2015sample} to be the correct measure of the inefficiency of proposal $q_t$. In particular, they showed that the number of samples needed for IS to work effectively scales as $\exp(KL(\pi|q))$. For proposals as given by \eqref{eq:proposalform}, this is,
\begin{align}
\KL(\pi\|q_t) &= \int_{\mathcal{X}} \pi(x) \log \frac{\pi(x)}{\sum_{a} q_{at} g_a(x)} dx\\
=\sum_a&\int_{\mathcal{X}_a} \pi(x)\log\frac{\pi(x)}{g_a(x)}dx - \sum_a \pi_a \log q_{at} \label{eq:kl}
\end{align}
This loss  takes an interesting form: the first term depends only on the local proposals $g_a(x)$ and the second only on
the subset probability masses $q_{at}$.  We thus see the surprising result that we can optimize separately
for $q_{at}$ and each $g_a(x)$. The closer $g_a(x)$ is to the target $\pi$ on subset $\mathcal{X}_a$, the less variability there is in the weights, and the smaller the variance factor, leading to lower KL divergence. We will discuss it in more details later in section \ref{hierarchy}. For now, we will only consider the problem of optimising $q_{at}$, for which we see
from~\eqref{eq:kl} has a true optimum
$q^*_{a} = \pi_a$ and therefore we define our regret $R(q_{t})$ from using proposal $q_t$ as
\begin{align}
R_t := R(q_t) = \KL(\pi\|q_t) - \KL(\pi\|q^*) = 
   \sum_a \pi_a \log\frac{\pi_a}{q_{at}},
   \label{eq:regret}
   \nonumber
\end{align}
which is just the KL divergence between the two probability vectors $(\pi_a)_{a=1}^K$ and $(q_{at})_{a=1}^K$.  Given this instantaneous regret of a
particular proposal, we can now use the cumulative regret
$\sum_{t=1}^T R_t$
as a performance measure of the adaptation scheme itself.  In particular, we would like schemes for which the regrets $R_t$ asymptote to zero, which corresponds to cumulative regrets that grow sublinearly.

Our regret analysis rests on an assumption that each $Y_{at}$ is \textbf{sub-Gaussian}.  Formally, there are variance factors $\tau_a^2$, which we assume known, such that $\EE[\exp(\theta(Y_{at} - \EE[Y_{at}]))] \leq \exp(\frac{1}{2}\tau_a^2\theta^2)$ for each $a$ and $\theta\in \RR$. 
This might seem like 
a strong assumption as importance weights are often wildly varying. However, in practice it can often be satisfied. For example, if the maximum  $M=\sup_{x\in\mathcal{X}} f(x)$ can be found, and we assume $\mathcal{X}$ is compact with each $g_a(x) = \mathbb{I}(x\in \mathcal{X}_a)/\int_{\mathcal{X}_a} dx$ uniform on its subset, then $\tau_a$ can be chosen as $\frac{M}{2}\int_{\mathcal{X}_a} dx$~\citep{WinNT}, which is $\ge \frac{Z_a}{2}$. In generally we expect $\tau_a$ of order $\mathcal{O}(Z_a)$.

\def\OO{\mathcal{O}}
\begin{theorem}
\label{theorem}
Assume sub-Gaussianity and define the optimism boost as $\sigma_{at} := c\tau_a\sqrt{{\log t}/{N_{at}}}$, where $c :=\sqrt{4.14\log_2(2e)}$. The cumulative pseudo-regret at $T$ of Daisee is bounded by $\OO(\sqrt{T}(\log T)^{\frac{3}{4}})$.
\end{theorem}
\begin{proof}
We show that the  pseudo-regret
$\EE[R_t] \le \EE[\sum_a\pi_a\log\frac{\pi_a}{{q}_{at}}]=\OO(t^{-\frac{1}{2}}(\log t)^{\frac{3}{4}})$; summing over $t=1,\ldots,T$ gives our result. The idea is to show that $\hat{Z}_{at}\approx Z_a$ and $N_{at}\approx \OO(t)$ with high probability, and that the pseudo-regret has the desired bound in this case. As such, define the events:
\begin{align*}
    B_t:&=\{|\hat{Z}_{as}-Z_{a}| < \sigma_{as} \;  \forall\; 1\leq a\leq K, \textstyle \frac{t}{2} \le s \leq t \}\\
    C_{t} :&= \{N_{at} > \beta_{at} \;\forall\; 1\le a\le K 
    \}, 
    \;\text{where }
    \beta_{at} := \textstyle\frac{tZ_a}{4\sum_b(Z_b + 2c\tau_b\sqrt{\log t})}
\end{align*}
The pseudo-regret can be written as below, and the proof proceeds by bounding each term separately:
\begin{equation}
    \EE[R_t] =\EE[R_t|B_t\cap C_t]\Prob(B_t\cap C_t)+ \EE[R_t|B_t^c]\Prob(B_t^c)+\EE[R_t|B_t\cap C_t^c]\Prob(B_t\cap C_t^c)
    \label{eq:decomposition}
\end{equation}

For the first term in \eqref{eq:decomposition},
this is just a bit of algebra: 
\begin{align}
&\EE[R_t|B_t\cap C_t]\Prob(B_t\cap C_t)
\leq  \EE[\textstyle \sum_a\pi_a\log\frac{\pi_a}{{q}_{at}} | B_t\cap C_t] 
\nonumber
\\
\leq &
\EE[\textstyle \sum_a\pi_a\log(\frac{\pi_a}{Z_a}\sum_b(Z_b + 2\sigma_{bt})) | B_t\cap C_t]
=
\EE[\textstyle \log(1+\frac{2}{Z}\sum_a\sigma_{at}) | B_t \cap C_t]   
\nonumber \\
\leq& \textstyle \frac{2}{Z}\sum_a\EE[\sigma_{at} |  B_t\cap C_{t}] 
\leq
\textstyle
\frac{2}{Z}\sum_a c\tau_a(\log t)^{\frac{1}{2}}\beta_{at}^{-\frac{1}{2}} 
=
\OO(\textstyle 
t^{-\frac{1}{2}}(\log t)^{\frac{3}{4}}
\frac{\sqrt{\sum_a \tau_a}}{\sum_a Z_a}\sum_a\frac{\tau_a}{\sqrt{Z_a}}
) 
\label{eq:eqn1}
\end{align}
For the second and third terms in \eqref{eq:decomposition},
the main steps are to show that both $\Prob(B_t^c)$ and $\Prob(C_t^c|B_t)$ are small. 
For $\Prob(B_t^c)$, for each $a$, we can rewrite $\hat{Z}_{as}$ as $\hat{Z}_{as} = \frac{1}{N_{as}}\sum_{i=1}^{N_{as}}Y_{at^a_i}$, where we re-index the time $\{t^a_i\}_{i=1}^{N_{as}} := \{j: A_j=a, 1\leq j \leq s\}$. Recall that each $Y_{at^a_i}$ is sub-Gaussian with mean $Z_a$ and variance factor $\tau_a^2$. The complication here is that there is a non-trivial dependence of these and the $N_{as}$'s through the proposal probabilities $q_{as}$'s. To side-step this dependence, we can use a finite time law of the iterated logarithm \citep{balsubramani2014sharp,wouterkoolen} (see Appendix \ref{app:lemmas})
to bound the deviations of $\hat{Z}_{as}$, for all possible values of $N_{as}$ between 1 and $s$. Choosing $\delta = 2s^{-2}$,  we have
\begin{align}
\textstyle    \sqrt{\frac{2.07\tau_a^2}{N_{as}} \log\left(\frac{2}{\delta}(1+\log_2N_{as})^2\right)}
\leq \sqrt{\frac{2.07\tau_a^2}{N_{as}} \left(2\log s + 2\frac{\log s}{\log 2}\right)} 
=\sigma_{as}
\end{align}
Hence,
\begin{align}
\textstyle
\mathbb{P}(|\hat{Z}_{as} - Z_a| > \sigma_{as}) 
& \textstyle
\leq \mathbb{P}\left(\exists l, 1 \leq l \leq s: \left|\frac{1}{l}\sum_{i=1}^{l} Y_{at^a_i} - Z_a\right| > \sqrt{\frac{2.07\tau_a^2}{l}\log \left(\frac{2}{\delta}(1+\log_2l)^2\right)}\right) \nonumber\\
&\leq \delta = 2s^{-2}
\end{align}
Using a union bound,
\begin{align}
\textstyle
\Prob(B_t^c) \leq \sum_{a=1}^K \sum_{s=[\frac{t}{2}]}^t \Prob(|\hat{Z}_{as} - Z_a| > \sigma_{as})
\le \sum_{a=1}^K \sum_{s=[\frac{t}{2}]}^t 2s^{-2}
= \OO(Kt^{-1})
\end{align}

To bound $\Prob(C_t^c|B_t)$, note that conditioned on $B_t$ we have, for $\frac{t}{2}\le s\le t$ and each $a$:
\begin{align}
\textstyle 
\sum_{s=1}^tq_{as} 
\geq \sum_{s=\frac{t}{2}}^t q_{as} 
\geq \sum_{s=\frac{t}{2}}^t \frac{Z_a}{\sum_b Z_b + 2\sigma_{bs}}
\geq \frac{Z_at}{2\sum_b(Z_b + 2 c\tau_b\sqrt{\log t})} =2\beta_{at}
\label{eq:beta}
\end{align}
We can write $N_{at} = \sum_{s=1}^t W_{as}$ where $W_{as}|\mathcal{F}_s \sim \text{Bernoulli}({q}_{as})$ and $\mathcal{F}_s$ is the information filtration up to iteration $s$. Since $(W_{as}-{q}_{as})_{s=1}^{\infty}$ is a martingale difference sequence, the Azuma-Hoeffding Inequality \citep{azuma1967weighted} (see Appendix \ref{app:lemmas}) gives:
\begin{align}
\Prob(N_{at} < \beta_{at} | B_t)
&\le \Prob(\textstyle N_{at} < \frac{1}{2}\sum_{s=1}^tq_{as} | B_t)
\leq  2\exp\left(-\frac{1}{2t}\beta_{at}^2\right)
\label{eq:boundN}
\end{align}
A union bound along with $e^{-\alpha}<1/\alpha$ then shows that $\Prob(C_t^c|B_t)\le \OO(K t^{-1}\log t)$.
Putting everything together, along with a loose bound of $\EE[R_t|B_t^c], \EE[R_t|C_t^c\cap B_t]\le \OO(\log t)$ (see Appendix \ref{app:proof}), gives the result that $\EE[R_t]$ is dominated by the first term, which is $ \OO(t^{-\frac{1}{2}} (\log t)^{\frac{3}{4}})$.
\end{proof}

It is worth elaborating on the dependence of \eqref{eq:eqn1} on $\tau_a$'s and $Z_a$'s. Recall that in general we expect $\tau_a= \mathcal{O}(Z_a)$.  Suppose $\tau_a\approx \Omega Z_a$ for some $\Omega>0$. Then 
$\frac{(\sum_a \tau_a)^{1/2}}{\sum_a Z_a}\sum_a\frac{\tau_a}{Z_a^{1/2}}
\approx \Omega^{3/2}\frac{\sum_a Z_a^{1/2}}{(\sum_a Z_a)^{1/2}}\le \Omega^{3/2}K^{1/2}
$.
The larger the variance factors (larger $\Omega$), and the larger the number of arms $K$, the worse the bound. Another harder scenario occurs when masses are not uniformly distributed across the arms, and we have no a priori knowledge of which arm has high mass.  That is, suppose $Z_\text{max}=Z_1\gg Z_a = Z_\text{min}>0$ for $a>1$, and $\tau_a\approx \Omega Z_\text{max}$ for all $a$.  Then 
$\frac{(\sum_a \tau_a)^{1/2}}{\sum_a Z_a}\sum_a\frac{\tau_a}{Z_a^{1/2}}
\approx  \Omega^{3/2}K^{3/2}(\frac{Z_\text{max}}{Z_{\text{min}}})^\frac{1}{2}$, so the difficulty of the problem is governed by the square root of the ratio $Z_\text{max}/Z_\text{min}$, and there is an extra factor of $K$ due to the higher regret of searching $K$ arms only to find just one of them has high mass.

In Appendix \ref{alpha_loss} we generalise the result to regrets defined by $\alpha$-divergence losses, obtaining a cumulative pseudo-regret bound of $ \OO(t^{-\frac{1}{2}} (\log t)^{\frac{1}{2}(\frac{1}{2\alpha}+1)})$. As $\alpha \rightarrow 1$, the $\alpha$-divergence reduces to the KL divergence and the regret bounds coincide. When $\alpha=2$ and $Z$ is known, the loss is $\frac{1}{2}\int_\mathcal{X}(\frac{\pi(x)}{q_t(x)}-1)^2 q_t(x)dx$, which is the variance of the importance weights.

\section{Experiments}
\label{sec:experiment}

\subsection{Daisee on 1D example}
\label{sec:exp_Daisee}
We first demonstrate Daisee on a simple problem, empirically evaluating a number of forms of the optimism boost.  Our sample space is the unit interval $[0,1]$ and we partitioned it evenly into 100 subintervals. Our target density $\pi(x)$ and the subproposals $g_a(x)$ were such that $Y_{at}$ has distribution $2a\text{Bernoulli}(1/100)/101$ for interval $a\in\{1,\ldots,100\}$.  We made this choice so that $Y_{at}$ has a large but controllable variance, so that the resulting problem of adaptation is hard enough for exploration to be important. $\tau_a$ are selected to be proportional to the volume of the subspace (same for all arms), and is tuned to minimise regret, after which we fixed its value and repeat the experiments. We did not find the regret very sensitive to the values of $\tau_a$. Figure \ref{fig:regret_KL} shows that our algorithm 
successfully recovers an effective proposal.
Figure \ref{fig:regret_boost} compares the cumulative regret for various possible optimism boosts, each of which involve an (not shown) constant multiplier that has been
optimised as a hyperparameter.  
 We observe that optimism boosts with an inverse relationship with $N_{at}$ and slow growth with $t$ work well, and seem to achieve sublinear cumulative regret. Among different forms of the boost, it can be observed that our defined optimism boost gives the lowest cumulative regret.

\begin{figure*}[h!]
    \centering
         \begin{subfigure}[b]{0.21\textwidth}
        \includegraphics[width=\textwidth]{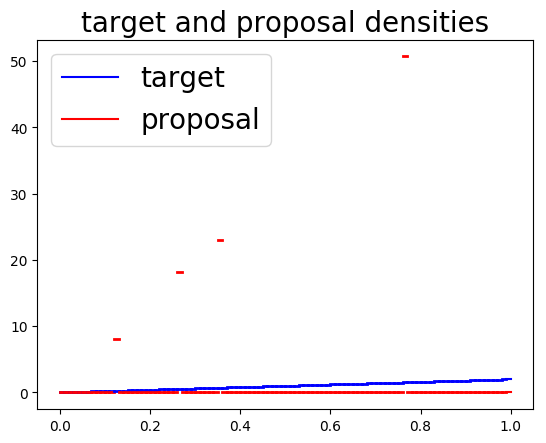}
        \caption{}
    \end{subfigure}
     \begin{subfigure}[b]{0.24\textwidth}
        \includegraphics[width=\textwidth]{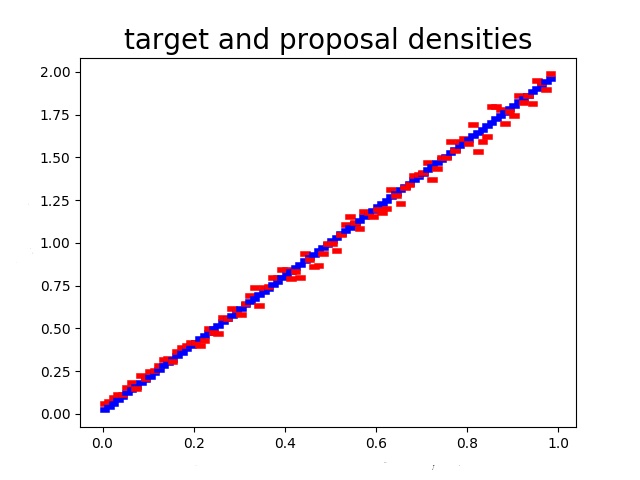}
        \caption{}
    \label{fig:regret_KL}
    \end{subfigure}
    \begin{subfigure}[b]{0.24\textwidth}
        \includegraphics[width=\textwidth]{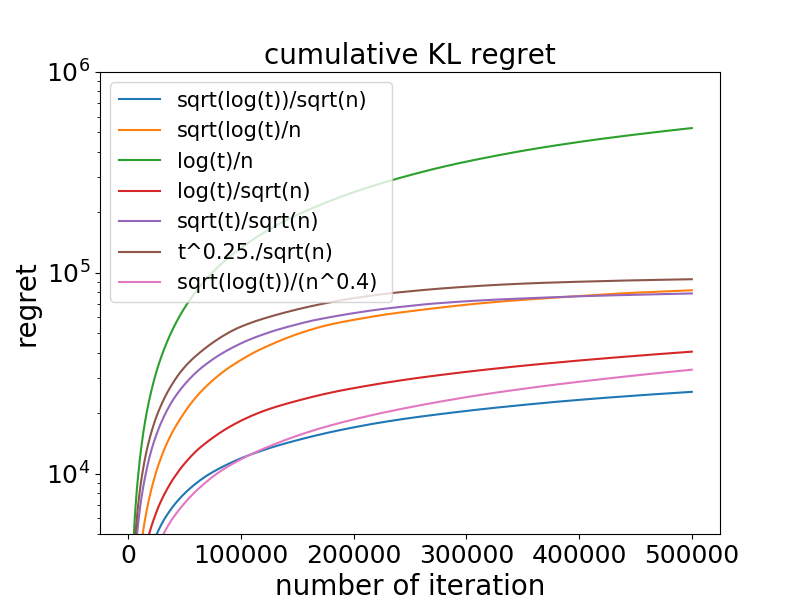}
        \caption{}
        \label{fig:regret_boost}
    \end{subfigure}
        \begin{subfigure}[b]{0.24\textwidth}
        \includegraphics[width=\textwidth]{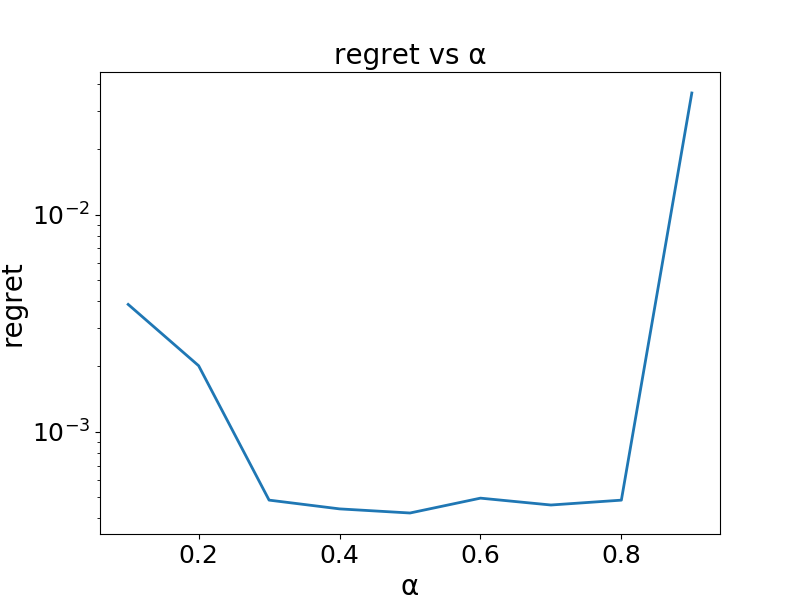}
        \caption{}
        \label{fig:regret_alpha}
    \end{subfigure}
    \caption{Results for 1d example. (a)(b) Target and proposal probabilities at final iteration, with no optimism boost and with optimism boost $\sqrt{\log (t)/N_{at}}$ respectively; (c) Cumulative regrets as functions of iteration, averaged over 10 runs for different forms of optimism of boost ($\sigma_{at}$).
    We observe that our chosen $\sigma_{at}$ outperforms the others; 
    (d) We explore the amount of exploration more systematically by considering boost of the form $\sigma_{at} = \left(\log (t)/N_{at}\right)^\alpha$ and reporting the final
    instantaneous regret as a function of $\alpha$, averaged over 10 runs. It can be observed that values of $\alpha$ near to $0.5$ gave the lowest regret, matching the theory.
    }
\end{figure*}
\begin{figure*}[ht]
    \centering
         \begin{subfigure}[b]{0.19\textwidth}
        \includegraphics[width=\textwidth]{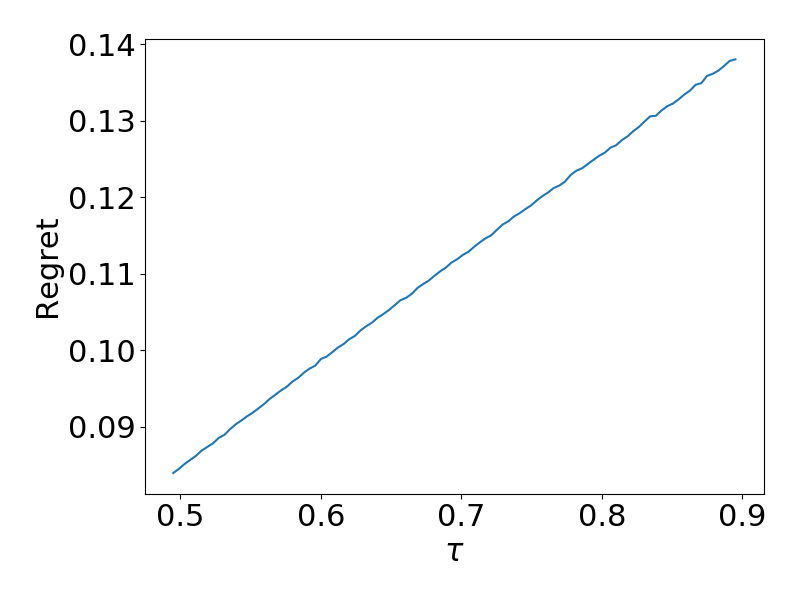}
        \caption{}
        \label{fig:vary_tau}
    \end{subfigure}
     \begin{subfigure}[b]{0.19\textwidth}
        \includegraphics[width=\textwidth]{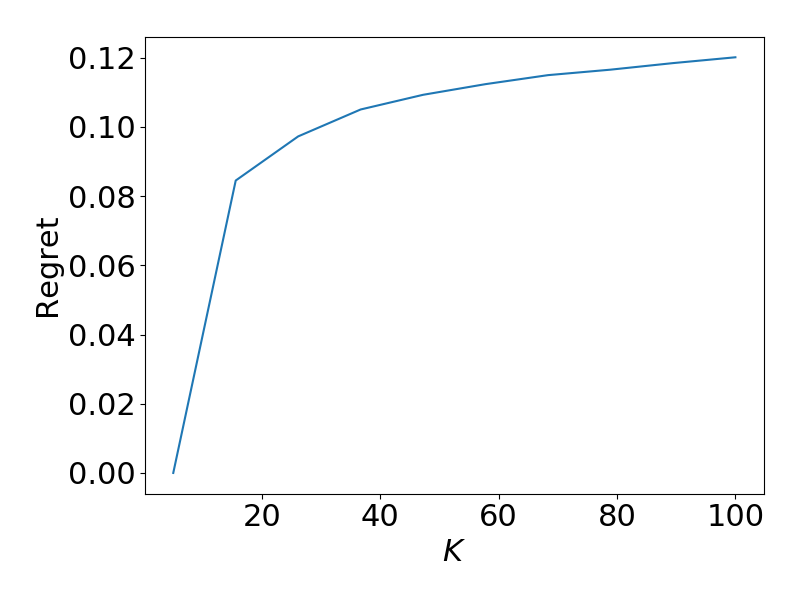}
        \caption{}
    \label{fig:vary_k}
    \end{subfigure}
    \begin{subfigure}[b]{0.19\textwidth}
        \includegraphics[width=\textwidth]{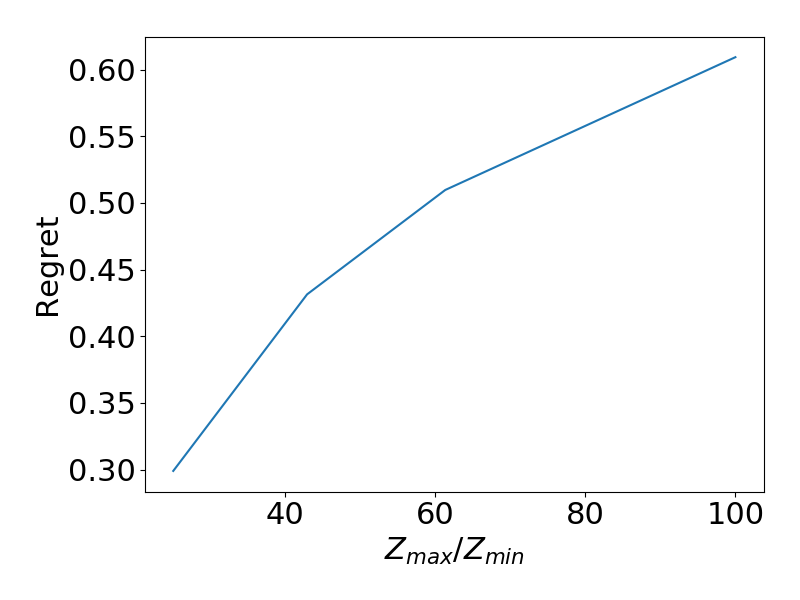}
        \caption{}
        \label{fig:vary_ratio}
    \end{subfigure}
        \begin{subfigure}[b]{0.19\textwidth}
        \includegraphics[width=\textwidth]{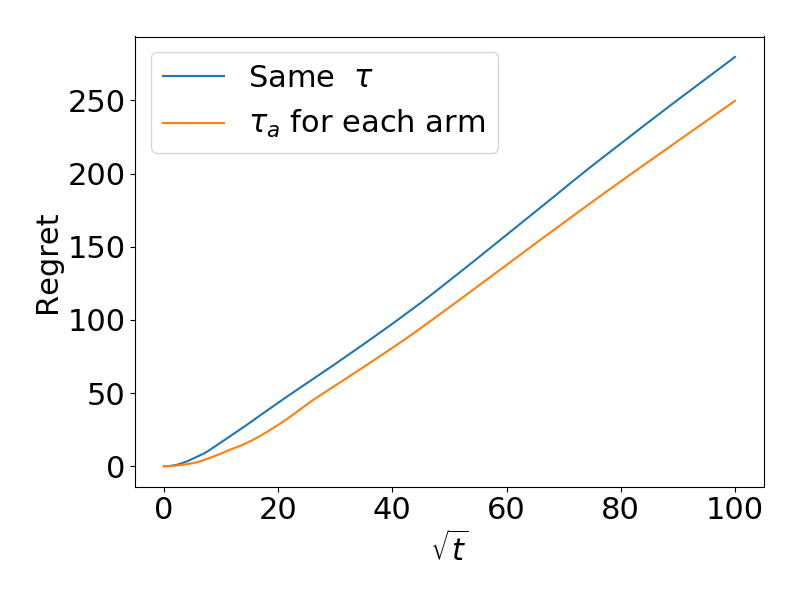}
        \caption{}
        \label{fig:each_tau}
    \end{subfigure}
     \begin{subfigure}[b]{0.19\textwidth}
        \includegraphics[width=\textwidth]{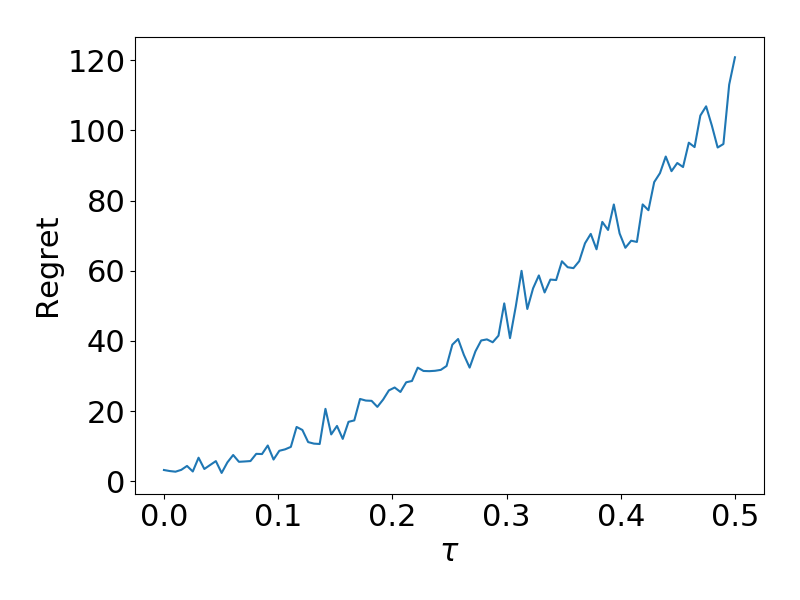}
        \caption{}
        \label{fig:sensitivity}
    \end{subfigure}
    \caption{(a)-(c) Cumulative regret versus $\tau$, $K$ and $\Z_{max}/Z_{min}$ respectively, while keeping the other parameters fixed. (d) Cumulative regret versus $\sqrt{t}$ when using the same $\tau$ or different $\tau_a$ for each arm. (e) Sensitivity analysis of cumulative regret on $\tau$ for a single target density.}
\end{figure*}

\subsection{Validation of Theorem \ref{theorem}}
In this section, we empirically evaluate the dependency of regret on different parameters, including the number of arms $K$, the variance factor $\tau$ and the ratio $Z_{max}/Z_{min}$ where $Z_{max} = \max_{a}\{Z_a\}_{a=1}^K$ and  $Z_{min} = \min_{a}\{Z_a\}_{a=1}^K$. Instead of tuning $\tau$ by grid search as in \ref{sec:exp_Daisee}, we assume it is known and we design the target density accordingly. We adapt the target density such that when varying each parameter, the other two are kept fixed. The sample space we consider is the interval $(0,1)$ and we partition the space into equi-spaced intervals. 

\textbf{Regret bound vs $\tau$:} We fix the number of arms $K=10$. The unnormalised density $f(x) = (10+\delta)\mathbbm{1}_{(0<x\leq 0.05)} +
(10-\delta)\mathbbm{1}_{(0.05<x\leq 0.1)} +
0.1\mathbbm{1}_{(0.1<x <1)}$, for $\delta$ ranging from $0.001$ to $8$ which results in a range of different variance factor $\tau$ while keeping $K$ and the ratio $Z_{max}/Z_{min}$ unchanged. It can be seen from Figure \ref{fig:vary_tau} that bigger $\tau$ leads to larger regret which coincides with our theorem. Intuitively, bigger $\tau$ indicated larger fluctuations which is harder and leads to larger regret.

\textbf{Regret bound vs $K$:} We set $f(x) = 3K\mathbbm{1}_{(0<x\leq 0.2)} + K\mathbbm{1}_{(0.2<x <1})$ and vary $K$ from $5$ to $100$. It can observed from figure \ref{fig:vary_k} that larger $K$ leads to larger regret, as derived from the theorem. This makes intuitive sense as it is easier to estimate relative masses over a smaller number of subspace.

\textbf{Regret bound vs $Z_{max}/Z_{min}$:}
In this case, we move some mass in the first arm equally to the rest of the arms.
$f(x) = 10\mathbbm{1}_{(0<x\leq 1/K-\delta)} +
0.1\mathbbm{1}_{(1/K-\delta <x\leq 1/K)} +
10\mathbbm{1}_{ (float(10x) < \delta/(K-1))} + 0.1 \mathbbm{1}_{(float(10x) \geq \delta/(K-1))}$
where $float$ indicates the floating part of the number. Again, the result from \ref{fig:vary_ratio} shows that the result is as expected according to the theorem. The bigger the ratio, the more variations there is in the target distributions, and more samples needs to be drawn to account for the uncertainties.

\textbf{Using separate $\tau_a$ on each arm:} If one knows each $\tau_a$ on arm $a$, then the regret should have a lower bound according to our derivation in (\ref{eq:eqn1}). We illustrate this on a simple problem where the target unnormalised density is $f(x) = 20\mathbbm{1}_{0<x\leq 0.25} +3\mathbbm{1}_{0.25<x\leq 0.5}
+9\mathbbm{1}_{0.5<x<\leq 0.99}
+\mathbbm{1}_{0.99<x < 1}$ and we set $K=5$. The result can be found in figure \ref{fig:each_tau} where with different $\tau_a$ according to the distribution over each arm leads to better performance. It can also be seen that the cumulative regret grows roughly linearly with $\sqrt{t}$, meaning our analysis for the regret is quite tight.

\textbf{Sensitivity analysis on $\tau$:} In practice, $\tau$ may not be estimated accurately and thus requires tuning. We tuned $\tau$ in the same way as in section \ref{sec:exp_Daisee} where it is chosen to minimise the regret, and show that the algorithm is robust for small values of $\tau$. With the same set up as in figure \ref{fig:each_tau}, we use the same $\tau$ for all arms and is tuned via grid search with $10$ repetitions, the resulting regret displayed in figure \ref{fig:sensitivity} verifies that Daisee is not very sensitive to $\tau$.

\begin{figure*}[ht]
    \centering
     \begin{subfigure}[b]{0.24\textwidth}
        \includegraphics[width=\textwidth]{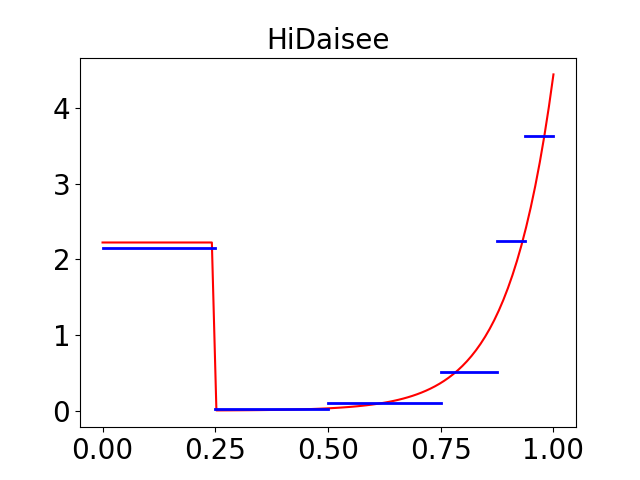}
                \caption{$100^{th}$ iteration}
    \end{subfigure}
    \begin{subfigure}[b]{0.24\textwidth}
        \includegraphics[width=\textwidth]{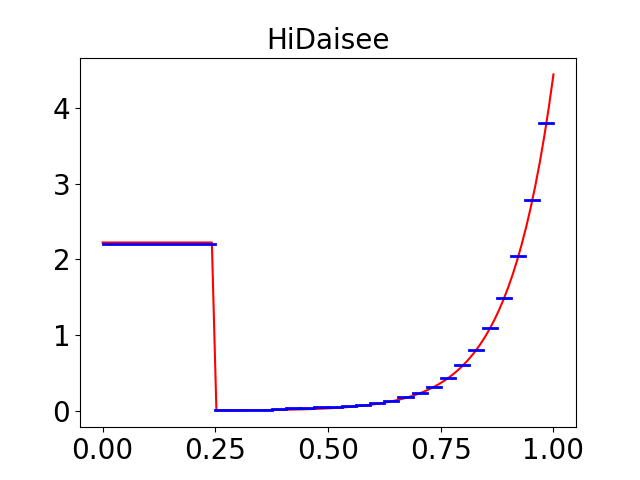}
                \caption{$10000^{th}$ iteration}
    \end{subfigure}
    \begin{subfigure}[b]{0.24\textwidth}
        \includegraphics[width=\textwidth]{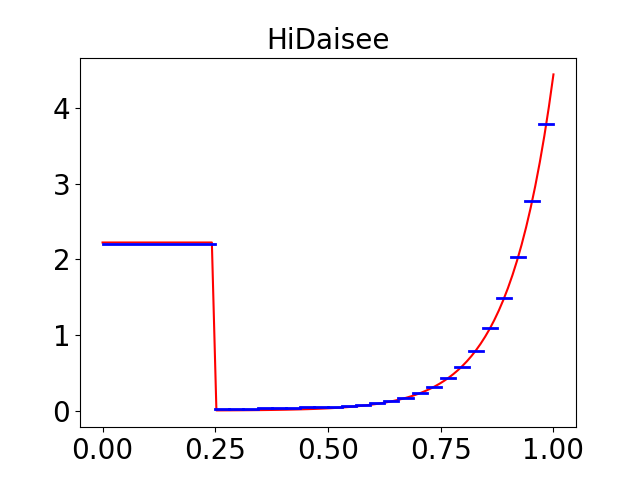}
                \caption{$100000^{th}$ iteration}
    \end{subfigure}
    \begin{subfigure}[b]{0.24\textwidth}
        \includegraphics[width=\textwidth]{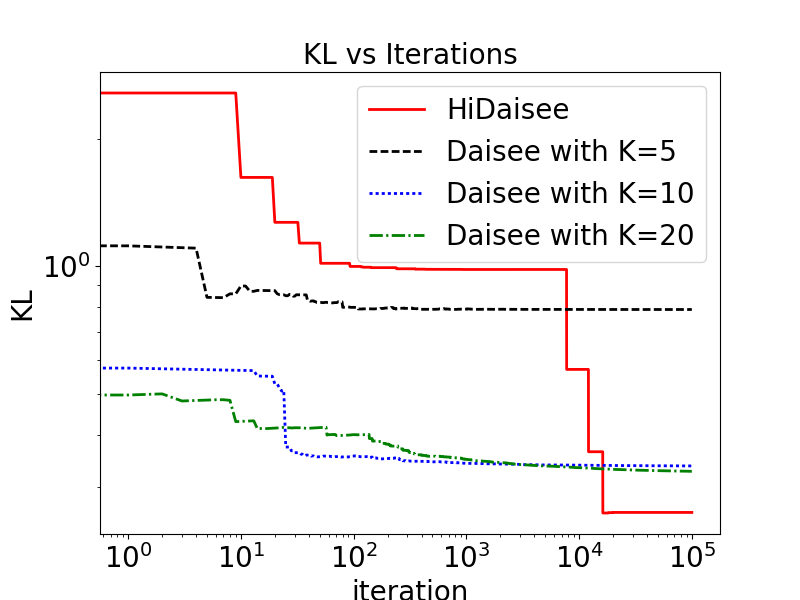}
        \caption{}
            \label{fig:hierarchy_compare}
    \end{subfigure}
    
         \begin{subfigure}[b]{0.24\textwidth}
        \includegraphics[width=\textwidth]{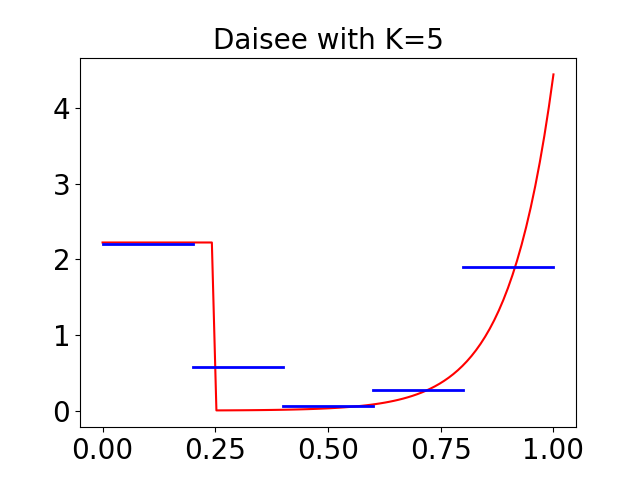}
                \caption{$100000^{th}$ iteration}
    \end{subfigure}
    \begin{subfigure}[b]{0.24\textwidth}
        \includegraphics[width=\textwidth]{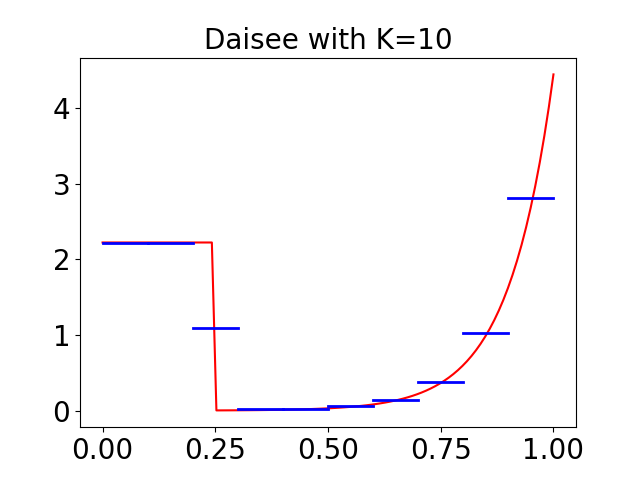}
                \caption{$100000^{th}$ iteration}
    \end{subfigure}
    \begin{subfigure}[b]{0.24\textwidth}
        \includegraphics[width=\textwidth]{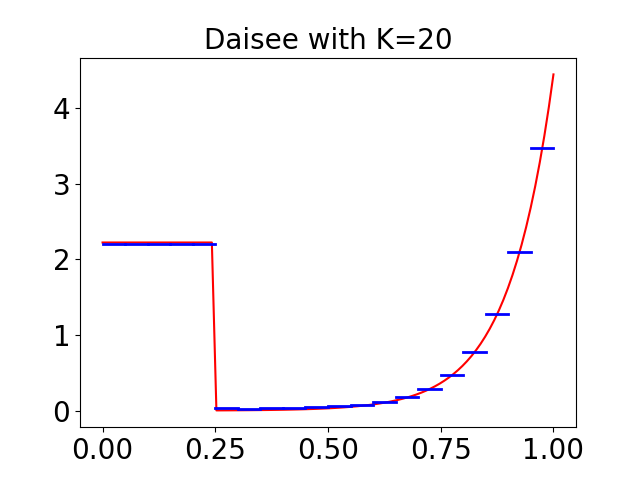}
                \caption{$100000^{th}$ iteration}
    \end{subfigure}
        \begin{subfigure}[b]{0.24\textwidth}
        \includegraphics[width=\textwidth]{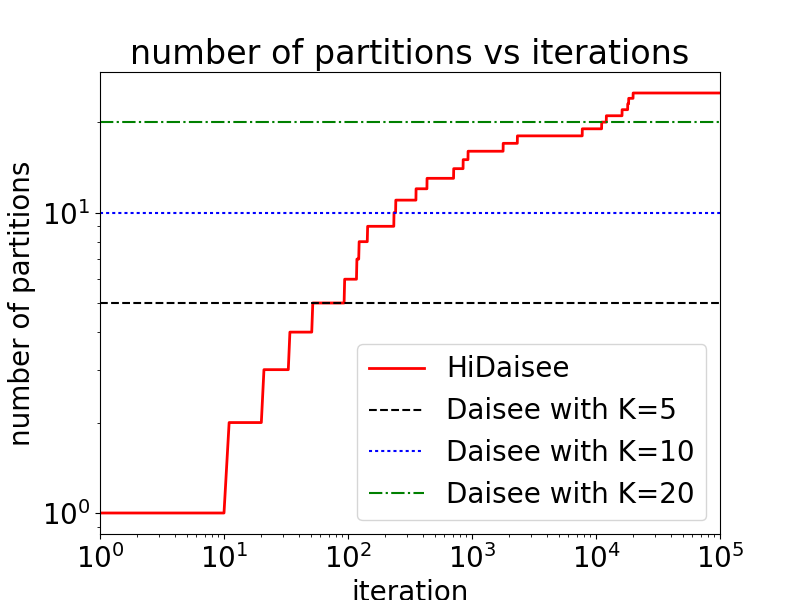}
        \caption{}
        \label{fig:HiDaisee_partition}
    \end{subfigure}
    \caption{(a)-(c): results for HiDaisee at iteration 1000, 10000 and 100000 respectively. Target densities are plotted in red and the adaptive proposal probabilities are plotted in blue. HiDaisee stops splitting the region where the density is relatively high but flat. (e)-(g): results for Daisee with different number of partitions $K = 5, 10, 20$, at iteration 100000, where the space has been partitioned into $K$ intervals of equal length. Last column shows the comparison of Daisee and HiDaisee. (d): $KL(\pi||q_t)$ vs iteration. Daisee with fixed partition asymtotes to the best proposal within that class, which can still be far from the true target density. For HiDaisee the KL continues to decrease. (h): Nmber of partitions vs iterations, which grows sub linearly with our splitting criterion.}
    \label{fig:hierarchy}
\end{figure*}

\section{Hierarchical Daisee}
\label{hierarchy}

We a view to improving the efficiency of Daisee, we now extend it by introducing an approach to refining partition proposals $g_a$ in part of the space where $g_a$ does not approximate $f$ well. Since we have a disjoint set of proposals $g_a$ on each arm $\mathcal{X}_a$, 
one approach of adapting the $g_a$ is to seek a better partitioning over the space. A good partitioning should be finely grained in areas where the density is highly fluctuating, and we consider adapting  the
partition alongside the subset probabilities $q_{at}$ using information from previous the samples. 
So, whereas Daisee used a set of fixed partitions, we will now adapt the partition by recursively splitting subsets into two halves.
We refer to this extended approach as \emph{HiDaisee} (Hierarchical Daisee). We now describe algorithmic details and show some promising experimental results for HiDaisee, but leave the theoretical analysis to future work.

For simplicity, we consider a finite rectangle for $\mathcal{X}$, and a hierarchical binary partition of $\mathcal{X}$ whereby each split simply splits a rectangle into two equal rectangles along one dimension (we simply cycle over the dimensions in the experiments). At each point of the algorithm the partition consists of rectangular subsets, and we use uniform subproposals $g_a(x)$. 
%
In order to learn the tree in an online fashion, HiDaisee decides whether to split a leaf node or not whenever
it is sampled.  We want to avoid oversplitting the tree as larger trees are more expensive to sample from and require more memory.
To control this, we introduce a \emph{splitting criterion}, such that whenever a node is chosen, it is split
if it passes the splitting criterion.  
We want to stop partitioning the subspace if the importance weights obtained from that subspace are similar to each other, 
as this indicates little is to be gained by further splitting.  To this end we introduce a splitting criterion based on the \emph{effective Sample Size} (ESS)~\citep{mcbook} of the node, a measure of efficiency in 
importance sampling that indicates the quality of the proposal. For rectangle $\mathcal{X}_a$ 
at iteration $t$, the $\text{ESS}_{at} \in [1,N_{at}]$ is defined as
\begin{align}
\text{ESS}_{at} := \frac{(\sum_{l\le t:A_l=a}Y_{al})^2}{\sum_{l\le t:A_l=a}Y_{al}^2}
\end{align}
where $Y_{al}$ are the local importance weights as before.
Our split criterion is then to only split when both a) the number of samples at the node exceeds a certain threshold $N_{\text{min}}$ (set to $10$ in
our experiments) and b) the 
ESS of the samples is less that $\alpha N_{at}$, where $\alpha$ is an ESS threshold parameter.

\begin{algorithm}[ht!]
\begin{algorithmic}[1]
\STATE \textbf{Input:} proposal $g(x)$, unnormalized target $f(x)$, ESS threshold $\alpha \in (0,1)$, minimum \#samples for splitting node $N_{\text{min}}$,
tree initialisation $\mathcal{T}$
\FOR {$t=1$ \TO $T$}
\STATE Set node id to root $i \leftarrow 0$ and initialise traversal path to be empty $\mathcal{P} \leftarrow \emptyset$
\WHILE{$i \notin \text{leaf nodes of } \mathcal{T}$}
\STATE $\mathcal{P} = \mathcal{P}\cup i$
\STATE $r \sim \text{Uniform}(0,1)$
\STATE $q_{\text{left}} = \frac{q_{i_{\ell}}}{q_{i_{\ell}}+q_{i_{r}}}$ \hfill \COMMENT{$i_{\ell}/i_{r}$ are left/right children}
\STATE {\bf if} $r < q_{\text{left}}$ {\bf then} $i \leftarrow i_{\ell}$ {\bf else} $i \leftarrow i_r$
{\bf end if}
\ENDWHILE
\STATE $x_t \sim g_i(x)$, compute weight $Y_{it} = f(x_t)/g_i(x_t)$
\STATE Update $q_{j}$ for all $j \in \mathcal{P}$ \hfill \COMMENT{Update $q_i$ for leaf as per~\eqref{eq:proposal} and
    all ancestors as per~\eqref{eq:q-inner}}
\STATE Update $\text{ESS}_{i}$ and $N_i$ for the leaf node
\IF{$N_i\ge N_{\text{min}}$ and $\text{ESS}_i < \alpha N_{i}$}
\STATE Split the node, updating $\mathcal{T}$ to include new nodes and calculating the corresponding $q_j$ by pushing samples down tree
\ENDIF
\ENDFOR
\end{algorithmic}
\caption{HiDaisee (Hierarchical Adaptive Importance Sampling with Exploration-Exploitation)}
\label{alg:HiDaisee}
\end{algorithm}

Though the leaves of the tree form a valid partitioning for Daisee and can thus be sampled directly,
the computational complexity of doing so na\"{i}vely scales as the number of subsets $K$ and so
becomes inefficient as the the tree becomes large.   By storing running estimates for each node $i$ in the tree, 
\begin{align}
\label{eq:q-inner}
    q_{it} = \sum\nolimits_{a \in \text{leaves($i$)}} q_{at}
\end{align} 
where the sum is over the leaves under node $i$, we can instead
traverse down the tree by recursively choosing the left or right child node until a leaf node is reached.
The computational cost now scales with the depth of the tree 
($\approx \mathcal{O}(\log K)$) rather than $K$.
Putting everything together, we arrive at the complete HiDaisee algorithm as shown in Algorithm~\ref{alg:HiDaisee},
where we will sometimes omit an implicit dependency on $t$ to avoid clutter.

\subsection{Experiment results for HiDaisee }
We first demonstrate HiDaisee on a simple example where the target density is $\pi(x) \propto \exp(10(x-1))\mathbf{1}_{x \in (0.25,1)}+ 0.5\mathbf{1}_{x \in (0,0.25]}$. 
Figure \ref{fig:hierarchy} shows the learned proposals at the different numbers of iterations, comparing HiDaisee to Daisee.
It can be seen that initially HiDaisee overestimates the region with low density to encourage exploration. The algorithm stops splitting the region on $0 < x < 0.25$ where the density is relatively high but is flat, whereas it continues to split for the rightmost region where the density is high and is not flat. A video of the learning evolution is available at \url{https://www.youtube.com/watch?v=LG5RCBcs4kg}.

Figure \ref{fig:hierarchy_compare} shows a comparison of  the KL loss per iteration for HiDaisee and Daisee for this example.
Though the initial performance of HiDaisee is worse than Daisee, because it initially has $K=1$, it quickly catches up and has the
best final KL loss.
It takes longer for HiDaisee to converge as it is a harder problem when starting with $K=1$. We also plot the evolution of the number of partitions in Figure \ref{fig:HiDaisee_partition}. Here it can be seen that the algorithm converges to a fixed partition, which is the optimal balancing between accuracy and time complexity. 

We also applied HiDaisee to the classic $2D$ banana shaped problem, with density  $f(x_1,x_2)\propto \exp\{-0.5(0.03x_1^2+(x_2+0.03(x_1^2-100))^2)\}$. The results are displayed in Figure \ref{fig:banana}. We see that the region with high density stops splitting when the density is relatively flat with our ESS splitting criterion. We compare HiDaisee with \emph{Parallel Interacting Markov Adaptive Importance Sampling} (PI-MAIS)~\cite{martino2017layered} for calculating the marginal likelihood (the normalising constant $Z$) in Figure \ref{fig:banana_compare}, which shows that when there are more samples and more complex hierarchies are constructed, HiDaisee achieves lower squared error.

\begin{figure*}[ht!]
\vspace{-5pt}
    \centering
    \begin{subfigure}[b]{0.25\textwidth}
        \includegraphics[width=\textwidth]{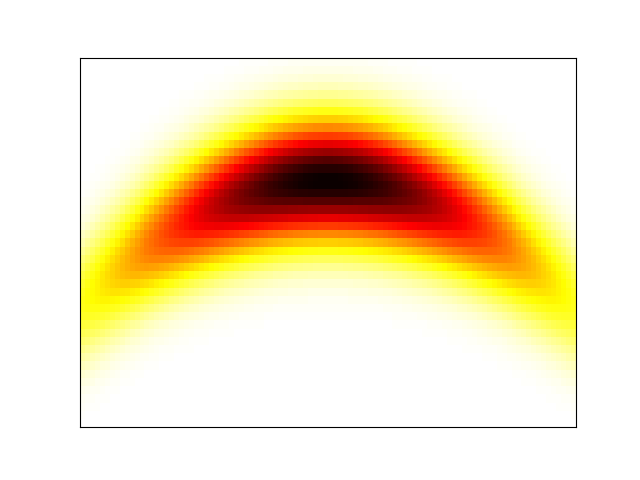}
        \caption{}
    \end{subfigure}
     \begin{subfigure}[b]{0.25\textwidth}
        \includegraphics[width=\textwidth]{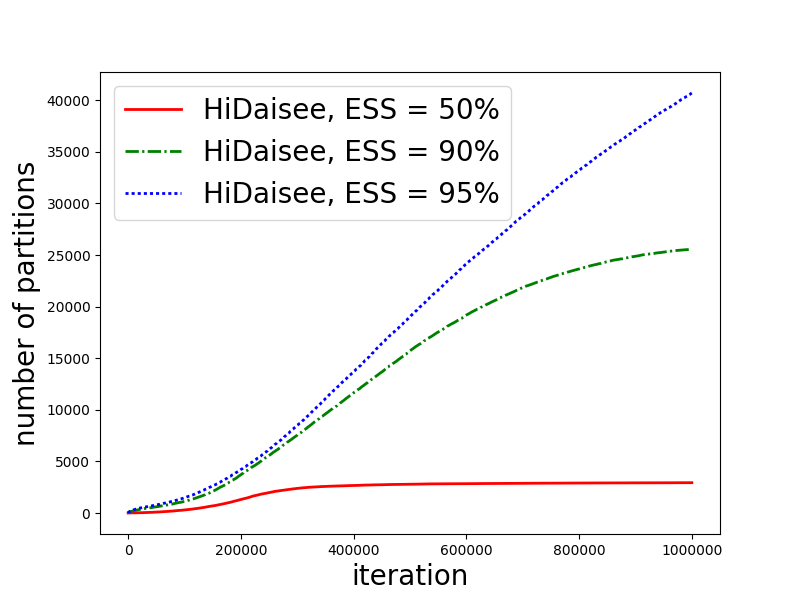}
        \caption{}
    \end{subfigure}
       \begin{subfigure}[b]{0.25\textwidth}
        \includegraphics[width=\textwidth]{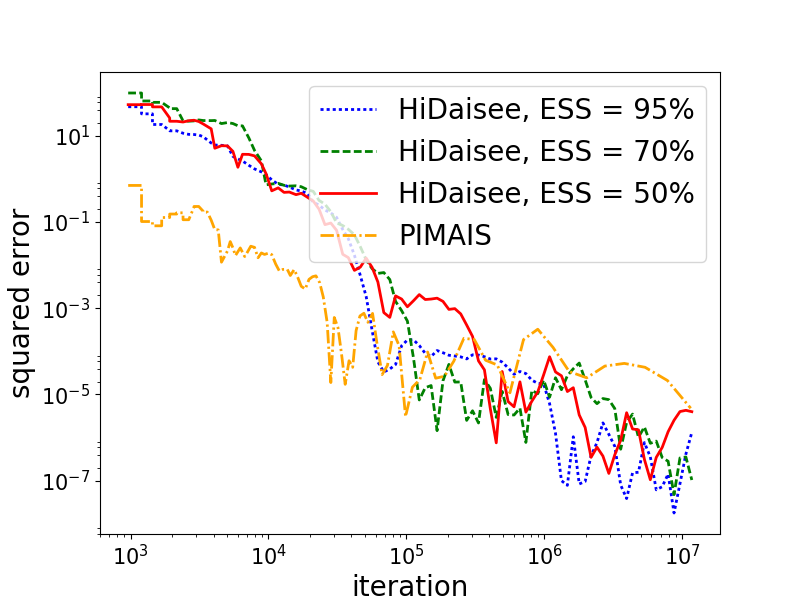}
        \caption{}
        \label{fig:banana_compare}
    \end{subfigure}  
    
     \begin{subfigure}[b]{0.25\textwidth}
        \includegraphics[width=\textwidth]{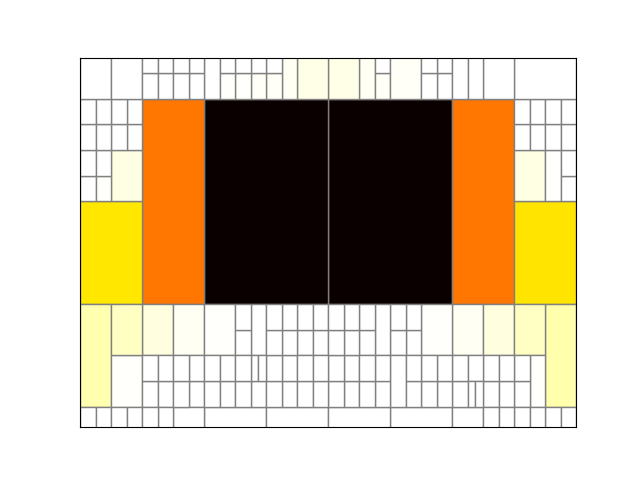}
        \caption{$ESS=50\%$}
    \end{subfigure}
    \begin{subfigure}[b]{0.25\textwidth}
        \includegraphics[width=\textwidth]{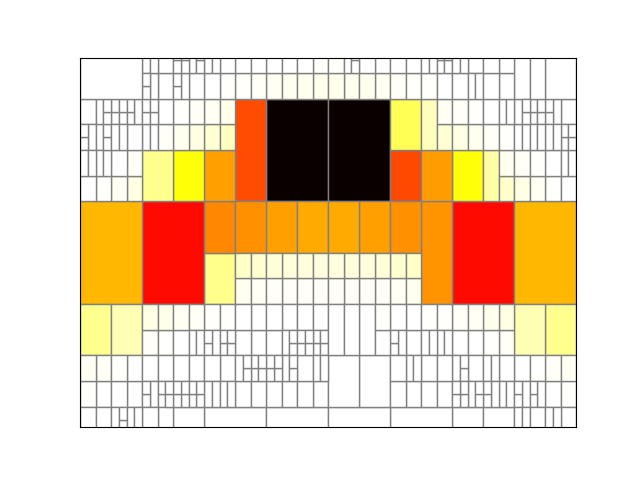}
        \caption{$ESS=70\%$}
    \end{subfigure}
    \begin{subfigure}[b]{0.25\textwidth}
        \includegraphics[width=\textwidth]{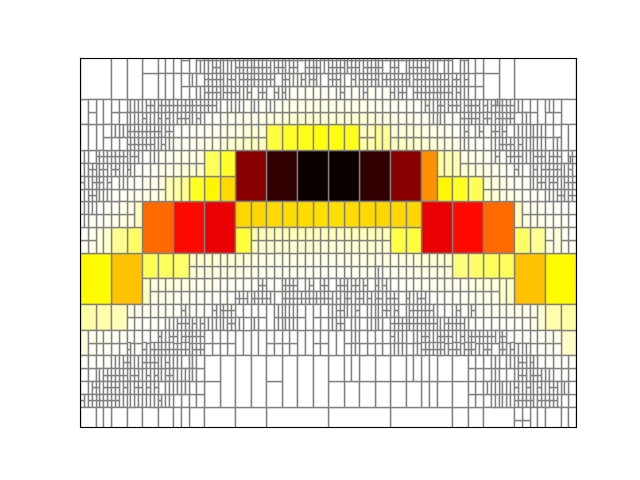}
        \caption{$ESS=95\%$}
    \end{subfigure}
    \caption{Banana shaped example, (a):target density; (b): number of partitions versus number of iterations for different ESS as a percentage of number of samples ;  (c): Comparision of HiDaisee and PI-MAIS in estimating the marginal likelihood $(\hat{Z} - Z)^2$, results are averaged over $10$ runs; (d)-(f) learned proposal distributions with difference ESS, showing the partitions. \vspace{-5pt}  
    }
    \label{fig:banana}
\end{figure*}

\section{Conclusions}
\label{sec:conclusion}
\vspace*{-0.2cm}

In this work, we have addressed the issue of exploration-exploitation in adaptive importance sampling, and proposed a novel approach through the lens of multi-armed bandit problems, borrowing the ideas of upper confidence bounds. We show a cumulative pseudo-regret of $\mathcal{O}(\sqrt{T}(\log T)^{\frac{3}{4}})$. We extend our method to the hierarchical case, where the sample space is recursively split in high density regions, and demonstrate experimentally that our methods give promising performance with little computational costs. Future work includes the relaxation of the sub-Gaussian assumption , an investigation of lower bounds on the KL regret in the AIS setting, and the theoretical analysis of (various variants of) HiDaisee. The analysis already introduced is an important step towards the analysis for HiDaisee, which requires substantial complications and is left as future work. It would also be interesting to investigate the application of the ideas we introduce to other adaptive MC methods.


\bibliographystyle{plainnat} 
\bibliography{ref}

\begin{thebibliography}{31}
\providecommand{\natexlab}[1]{#1}
\providecommand{\url}[1]{\texttt{#1}}
\expandafter\ifx\csname urlstyle\endcsname\relax
  \providecommand{\doi}[1]{doi: #1}\else
  \providecommand{\doi}{doi: \begingroup \urlstyle{rm}\Url}\fi

\bibitem[Agrawal and Goyal(2012)]{agrawal2012analysis}
Shipra Agrawal and Navin Goyal.
\newblock Analysis of thompson sampling for the multi-armed bandit problem.
\newblock In \emph{COLT}, pages 39--1, 2012.

\bibitem[Auer et~al.(2002)Auer, Cesa-Bianchi, and Fischer]{auer2002finite}
Peter Auer, Nicolo Cesa-Bianchi, and Paul Fischer.
\newblock Finite-time analysis of the multiarmed bandit problem.
\newblock \emph{Machine learning}, 47\penalty0 (2-3):\penalty0 235--256, 2002.

\bibitem[Azuma(1967)]{azuma1967weighted}
Kazuoki Azuma.
\newblock Weighted sums of certain dependent random variables.
\newblock \emph{Tohoku Mathematical Journal, Second Series}, 19\penalty0
  (3):\penalty0 357--367, 1967.

\bibitem[Balsubramani(2014)]{balsubramani2014sharp}
Akshay Balsubramani.
\newblock Sharp finite-time iterated-logarithm martingale concentration.
\newblock \emph{arXiv preprint arXiv:1405.2639}, 2014.

\bibitem[Berry and Fristedt(1985)]{berry1985bandit}
Donald~A Berry and Bert Fristedt.
\newblock \emph{Bandit problems: sequential allocation of experiments
  (Monographs on statistics and applied probability)}.
\newblock Springer, 1985.

\bibitem[Bubeck et~al.(2012)Bubeck, Cesa-Bianchi, et~al.]{bubeck2012regret}
S{\'e}bastien Bubeck, Nicolo Cesa-Bianchi, et~al.
\newblock Regret analysis of stochastic and nonstochastic multi-armed bandit
  problems.
\newblock \emph{Foundations and Trends{\textregistered} in Machine Learning},
  5\penalty0 (1):\penalty0 1--122, 2012.

\bibitem[Bubeck et~al.(2013)Bubeck, Cesa-Bianchi, and
  Lugosi]{bubeck2013bandits}
S{\'e}bastien Bubeck, Nicolo Cesa-Bianchi, and G{\'a}bor Lugosi.
\newblock Bandits with heavy tail.
\newblock \emph{IEEE Transactions on Information Theory}, 59\penalty0
  (11):\penalty0 7711--7717, 2013.

\bibitem[Bugallo et~al.(2017)Bugallo, Elvira, Martino, Luengo, Miguez, and
  Djuric]{bugallo2017adaptive}
Monica~F Bugallo, Victor Elvira, Luca Martino, David Luengo, Joaquin Miguez,
  and Petar~M Djuric.
\newblock Adaptive importance sampling: the past, the present, and the future.
\newblock \emph{IEEE Signal Processing Magazine}, 34\penalty0 (4):\penalty0
  60--79, 2017.

\bibitem[Capp{\'e} et~al.(2004)Capp{\'e}, Guillin, Marin, and
  Robert]{cappe2004population}
Olivier Capp{\'e}, Arnaud Guillin, Jean-Michel Marin, and Christian~P Robert.
\newblock Population monte carlo.
\newblock \emph{Journal of Computational and Graphical Statistics}, 13\penalty0
  (4):\penalty0 907--929, 2004.

\bibitem[Capp{\'e} et~al.(2008)Capp{\'e}, Douc, Guillin, Marin, and
  Robert]{cappe2008adaptive}
Olivier Capp{\'e}, Randal Douc, Arnaud Guillin, Jean-Michel Marin, and
  Christian~P Robert.
\newblock Adaptive importance sampling in general mixture classes.
\newblock \emph{Statistics and Computing}, 18\penalty0 (4):\penalty0 447--459,
  2008.

\bibitem[Carpentier and Munos(2011)]{carpentier2011finite}
Alexandra Carpentier and R{\'e}mi Munos.
\newblock Finite time analysis of stratified sampling for monte carlo.
\newblock In \emph{Advances in Neural Information Processing Systems}, pages
  1278--1286, 2011.

\bibitem[Chatterjee and Diaconis(2015)]{chatterjee2015sample}
Sourav Chatterjee and Persi Diaconis.
\newblock The sample size required in importance sampling.
\newblock \emph{arXiv preprint arXiv:1511.01437}, 2015.

\bibitem[Cichocki and Amari(2010)]{cichocki2010families}
Andrzej Cichocki and Shun-ichi Amari.
\newblock Families of alpha-beta-and gamma-divergences: Flexible and robust
  measures of similarities.
\newblock \emph{Entropy}, 12\penalty0 (6):\penalty0 1532--1568, 2010.

\bibitem[Cornuet et~al.(2012)Cornuet, Marin, Mira, and
  Robert]{cornuet2012adaptive}
Jean Cornuet, Jean-Michel Marin, Antonietta Mira, and Christian~P Robert.
\newblock Adaptive multiple importance sampling.
\newblock \emph{Scandinavian Journal of Statistics}, 39\penalty0 (4):\penalty0
  798--812, 2012.

\bibitem[Friedman(1991)]{friedman1991multivariate}
Jerome~H Friedman.
\newblock Multivariate adaptive regression splines.
\newblock \emph{The annals of statistics}, pages 1--67, 1991.

\bibitem[He and Owen(2014)]{he2014optimal}
Hera~Y He and Art~B Owen.
\newblock Optimal mixture weights in multiple importance sampling.
\newblock \emph{arXiv preprint arXiv:1411.3954}, 2014.

\bibitem[Johnson-Lindenstrauss()]{WinNT}
Johnson-Lindenstrauss.
\newblock {Johnson-Lindenstrauss theory}, url =
  {http://lear.inrialpes.fr/people/harchaoui/teaching/2013-2014/ensl/m2/lecture6.pdf}.

\bibitem[Koolen(2017)]{wouterkoolen}
Wouter~M. Koolen.
\newblock A quick and dirty finite time law of the iterated logarithm result.
\newblock \emph{http://blog.wouterkoolen.info/QnD\_LIL/post.html}, 2017.

\bibitem[Kschischang(2017)]{frank}
Frank~R. Kschischang.
\newblock The complementary error function.
\newblock \emph{http://www.comm.utoronto.ca/frank/notes/erfc.pdf}, 2017.

\bibitem[Lai and Robbins(1985)]{lai1985asymptotically}
Tze~Leung Lai and Herbert Robbins.
\newblock Asymptotically efficient adaptive allocation rules.
\newblock \emph{Advances in applied mathematics}, 6\penalty0 (1):\penalty0
  4--22, 1985.

\bibitem[Lepage(1978)]{lepage1978new}
G~Peter Lepage.
\newblock A new algorithm for adaptive multidimensional integration.
\newblock \emph{Journal of Computational Physics}, 27\penalty0 (2):\penalty0
  192--203, 1978.

\bibitem[Lepr{\^e}tre et~al.(2017)Lepr{\^e}tre, Teytaud, and
  Dehos]{lepretre2017multi}
Florian Lepr{\^e}tre, Fabien Teytaud, and Julien Dehos.
\newblock Multi-armed bandit for stratified sampling: Application to numerical
  integration.
\newblock In \emph{TAAI 2017-Conference on Technologies and Applications of
  Artificial Intelligence}, 2017.

\bibitem[Liu(2008)]{liu2008monte}
Jun~S Liu.
\newblock \emph{Monte Carlo strategies in scientific computing}.
\newblock Springer Science \& Business Media, 2008.

\bibitem[Martino et~al.(2017)Martino, Elvira, Luengo, and
  Corander]{martino2017layered}
Luca Martino, Victor Elvira, David Luengo, and Jukka Corander.
\newblock Layered adaptive importance sampling.
\newblock \emph{Statistics and Computing}, 27\penalty0 (3):\penalty0 599--623,
  2017.

\bibitem[Neufeld et~al.(2014)Neufeld, Gy{\"o}rgy, Schuurmans, and
  Szepesv{\'a}ri]{neufeld2014adaptive}
James Neufeld, Andr{\'a}s Gy{\"o}rgy, Dale Schuurmans, and Csaba
  Szepesv{\'a}ri.
\newblock Adaptive monte carlo via bandit allocation.
\newblock \emph{arXiv preprint arXiv:1405.3318}, 2014.

\bibitem[Owen and Zhou(2000)]{owen2000safe}
Art Owen and Yi~Zhou.
\newblock Safe and effective importance sampling.
\newblock \emph{Journal of the American Statistical Association}, 95\penalty0
  (449):\penalty0 135--143, 2000.

\bibitem[Owen(2013)]{mcbook}
Art~B. Owen.
\newblock \emph{Monte Carlo theory, methods and examples}.
\newblock 2013.

\bibitem[Owen et~al.(2017)Owen, Maximov, and Chertkov]{owen2017importance}
Art~B Owen, Yury Maximov, and Michael Chertkov.
\newblock Importance sampling the union of rare events with an application to
  power systems analysis.
\newblock \emph{arXiv preprint arXiv:1710.06965}, 2017.

\bibitem[Rainforth et~al.(2018)Rainforth, Zhou, Lu, Teh, Wood, Yang, and van~de
  Meent]{rainforth2018inference}
Tom Rainforth, Yuan Zhou, Xiaoyu Lu, Yee~Whye Teh, Frank Wood, Hongseok Yang,
  and Jan-Willem van~de Meent.
\newblock Inference trees: Adaptive inference with exploration.
\newblock \emph{arXiv preprint arXiv:1806.09550}, 2018.

\bibitem[Sen et~al.(2018)Sen, Shanmugam, and Shakkottai]{sen2018contextual}
Rajat Sen, Karthikeyan Shanmugam, and Sanjay Shakkottai.
\newblock Contextual bandits with stochastic experts.
\newblock \emph{arXiv preprint arXiv:1802.08737}, 2018.

\bibitem[Srinivas et~al.(2009)Srinivas, Krause, Kakade, and
  Seeger]{srinivas2009gaussian}
Niranjan Srinivas, Andreas Krause, Sham~M Kakade, and Matthias Seeger.
\newblock Gaussian process optimization in the bandit setting: No regret and
  experimental design.
\newblock \emph{arXiv preprint arXiv:0912.3995}, 2009.

\end{thebibliography}

\clearpage
\appendix

\appendix

\section{Concentration Inequalities}\label{app:lemmas}

\begin{theorem}{(Finite Time Law of the Iterated Logarithm \citep{balsubramani2014sharp,wouterkoolen})}
Let $X_1, X_2,\ldots  $ be a sequence of iid sub-Gaussian random variables with mean $\mu$ and variance factor $\tau^2$. For each $\delta \in (0,1)$, we have
\begin{align}
\mathbb{P}\left(\exists s \ge 1, \left|\frac{1}{s}\sum_{i=1}^s X_i - \mu\right| \geq 
\sqrt{\frac{2.07\tau^2}{s}\log \left(\frac{2}{\delta}(1+\log_2s)^2\right)}\right) \leq \delta
 \end{align}   
\label{balsubramani2014sharp}
\end{theorem}
\begin{theorem}{(Azuma–Hoeffding Inequality \citep{azuma1967weighted})}.
Let $D_1, D_2,\ldots$ be a martingale difference sequence, and suppose that $|D_k| \leq b_k$ almost surely for all $k \geq 1$. Then for all $n \geq 0, \Prob(|\sum_{i=1}^nD_i| \geq \epsilon) \leq 2\exp\left(-\frac{\epsilon^2}{2\sum_{i=1}^n b_i^2}\right)$.
\label{Hoeffding_MDS}
\end{theorem}

\section{Proof Details of Theorem \ref{theorem}}
\label{app:proof}

Note that for any event $E$, by Jensen's inequality,
\begin{align}
    \EE[R_t |E] \leq \sum_a\pi_a \log \left(\pi_a \EE\left[\frac{\sum_b (\hat{Z}_{bt} + \sigma_{bt})}{\hat{Z}_{at} + \sigma_{at}}|E\right]\right)
    \label{eq:eq4}
\end{align}

\textbf{Bounding $\EE[R_t|C_t^c \cap B_t]:$}

\begin{align}
\EE[R_t|C_t^c \cap B_t] &\leq \sum_a \pi_a \log\left(\pi_a\EE\left[\frac{\sum_b (Z_b + 2\sigma_{bt})}{Z_a} | C_t^c \cap B_t \right]\right)\\
&\leq
\frac{2}{Z}\sum_a\EE[\sigma_{at} |  C_t^c \cap B_t]  \\
&= \frac{2}{Z}\sum_a \EE\left[c\tau_a\sqrt{\frac{\log t}{N_{at}}}| C_t^c\right]\\
&\leq \frac{2c\sum_a\tau_a}{Z}\sqrt{\log t} 
\end{align}

\textbf{Bounding $\EE[R_t| B_t^c]:$} recall (\ref{eq:eq4}), we first bound

\begin{align}
    \EE\left[\frac{\sum_b (\hat{Z}_{bt} + \sigma_{bt})}{\hat{Z}_{at} + \sigma_{at}}| B_t^c\right]
    \leq \EE\left[\frac{\sum_b (\hat{Z}_{bt} + \sigma_{bt})}{ \sigma_{at}}| B_t^c\right]
    \leq \frac{\sqrt{t}}{c\tau_a\sqrt{\log t}}\sum_b\EE[ (\hat{Z}_{bt} + \sigma_{bt})| B_t^c]
    \label{eq:eq5}
\end{align}

Now 
\begin{align}
    \EE[ (\hat{Z}_{bt} + \sigma_{bt})| B_t^c] \leq 
    \EE[ (\hat{Z}_{bt} + \sigma_{bt})| B_t^c, |\hat{Z}_{bt} - Z_b| > \sigma_{bt}] + 
    \EE[ (\hat{Z}_{bt} + \sigma_{bt})| B_t^c, |\hat{Z}_{bt} - Z_b| \leq \sigma_{bt}]
    \nonumber
\end{align}
where the second term is bounded by $Z_b + 2c\tau_b\sqrt{\log t}$. For the first term, we use properties of  sub-Gaussian distributions:
\begin{align}
    \EE[ (\hat{Z}_{bt} + \sigma_{bt})&| B_t^c, |\hat{Z}_{bt} - Z_b| > \sigma_{bt}] 
    \leq \sum_{n=1}^t  \EE[ \hat{Z}_{bt} + \sigma_{bt}| |\hat{Z}_{bt} - Z_b| > \sigma_{bt}, N_{bt} = n]\\
    &= \sum_{n=1}^t \EE\left[\frac{1}{n}\sum_{j=1}^n (Y_{bt_j}^b - Z_b) + Z_b + c\tau_b\frac{\sqrt{\log t}}{n} | \frac{1}{n}\sum_{j=1}^n (Y_{bt_j}^b - Z_b) >  c\tau_b\frac{\sqrt{\log t}}{n} \right] \\
    &\leq \sum_{n=1}^t \EE[W_n | W_n > c\tau_b\sqrt{\log t}] + t(Z_b + c\tau_b\sqrt{\log t})
    \label{eq:eq6}
\end{align}

where $W_n: = \frac{1}{n}\sum_{j=1}^n (Y_{bt_j}^b - Z_b) $ is also sub Gaussian distributed with zero mean and variance proxy $\frac{\tau_b^2}{n}$ since it is a linear combination of $n$ sub Gaussian distributed random variables. For large $t$ the mean of the truncated sub Gaussian $\EE[W_n | W_n > c\tau_b\sqrt{\log t}] $ is upper bounded by that of a truncated Gaussian:

\begin{align}
    \EE[W_n | W_n > c\tau_b\sqrt{\log t}]  &\leq \tau \frac{\phi(c\tau_b\sqrt{\log t})}{1-\Phi(c\tau_b\sqrt{\log t})} 
    =\frac{2\phi(c\tau_b\sqrt{\log t})}{\text{erfc}(\frac{c\tau_b\sqrt{\log t}}{\sqrt{2}})} \\
    &\leq \frac{\phi(c\tau_b\sqrt{\log t})}{\exp(-\frac{c^2\tau_b^2\log t}{2})} \sqrt{\pi} \left(\frac{c\tau_b\sqrt{\log t}}{\sqrt{2}} + \sqrt{\frac{c^2\tau_b^2\log t}{2} + 2}\right)
    &\leq \sqrt{2}c\tau_b\sqrt{\log t}
    \label{eq:eq7}
\end{align}
where $\phi(\cdot)$ and $\Phi(\cdot)$ are the $pdf$ and $cdf$ of a standard normal distributions and $\text{erfc}$ denotes the complementary error function, we used \cite{frank} for the inequality involving $\text{erfc}$ above. Hence we have $        \EE[ (\hat{Z}_{bt} + \sigma_{bt})| B_t^c] \leq \mathcal{O}(t\sqrt{\log t}) $ and combinaing (\ref{eq:eq4})(\ref{eq:eq5})(\ref{eq:eq6})(\ref{eq:eq7}) gives

\begin{align}
\EE[R_t| B_t^c] \leq \sum_a\pi_a\log(\frac{\pi_a}{c}\sqrt{\frac{t}{\log t}}\sum_b (\sqrt{2}c\tau_bt\sqrt{\log t} + Z_b) \leq \mathcal{O}(\log t)
\end{align}

\section{Alpha Loss Family}
\label{alpha_loss}

We can generalise the KL loss to the family of alpha-loss which is associated with $\alpha-$ divergence for $\alpha\in\RR$\cite{cichocki2010families}. For any target (unnormalised) density $\p$ and proposal $\ \sim q$, define:
\begin{align*}
    &\mathcal{L}_\alpha (x,q) = \frac{1}{\alpha(1-\alpha)}\left(\alpha\frac{\pi(x)}{q(x)} +     (1-\alpha) - \left(\frac{\pi(x)}{q(x)}\right)^\alpha\right) \\
\end{align*}
where we take the limit for $\alpha = 0, 1$. Taking expectation with respect to $x$ yields the $\alpha-$divergence between $p$ and $q$:
\begin{align*}
    \mathcal{L}_\alpha (q) &= \mathbf{E}[\mathcal{L}_\alpha (x,q)]= D_\alpha(\pi||q)
            =\frac{1}{1-\alpha}Z + \frac{1}{\alpha}-\frac{1}{\alpha(1-\alpha)}
            \int_\mathcal{X}\left(\frac{\pi(x)}{q(x)}\right)^\alpha q(x)dx \\
\end{align*}
In particular, $\mathcal{L}_1 (q) = KL(\pi||q), \mathcal{L}_0 (q) = KL(q||\pi)$ and when $\alpha=2$, this leads to the $\mathcal{L}_2$ loss:
$    \mathcal{L}_2(q) = D_2(p||q)=
    \frac{1}{2}\int_\mathcal{X}\left(\frac{\pi(x)}{q(x)}-1\right)^2q(x)dx.
$
Using Lagrangian gives the optimal proposal
\begin{align}
    q^*_{a} = \frac{\left(\int_{\mathcal{X}_a}\pi(x)^\alpha g_a(x)^{1-\alpha}dx\right)
              ^\frac{1}{\alpha}}
               {\sum_a\left(\int_{\mathcal{X}_a}\pi(x)^\alpha g_a(x)^{1-\alpha}dx\right)
              ^\frac{1}{\alpha}}
              := \frac{(\pi^{(\alpha)}_a)^\frac{1}{\alpha}}{\sum_a(\pi^{(\alpha)}_a)^\frac{1}{\alpha}}
    \label{eq:alpha_q}
\end{align}
where $\pi^{(\alpha)}_a := \int_{\mathcal{X}_a}\pi(x)^\alpha g_a(x)^{1-\alpha}dx$.
The optimal loss is therefore
$\mathcal{L}_\alpha(q^*) =  \frac{1}{1-\alpha} + \frac{1}{\alpha}+\frac{1}{\alpha(\alpha-1)}(\sum_a \pi^{(\alpha)}_a)^\alpha$. Similarly as in section \ref{sec:AdaIS}, we can use a Monte Carlo estimate for the intractable optimal $q^*$:
\begin{align}
    q^*_{a} \propto \pi^{(\alpha)}_a = \frac{1}{Z}\int_{\mathcal{X}_a}\left(\frac{f(x)}{g_a(x)}\right)^\alpha g_a(x)dx
     \propto \EE_{g_a}\left[\left(\frac{f(x)}{g_a(x)}\right)^\alpha\right]
\end{align}

Replacing the definition of $Y_{al}$ with $Y_{al} := \left(\frac{f(x_l)}{g_a(x_l)}\right)^\alpha$, and refine $\hat{Z}_{a,t+1}$ as in (\ref{eq:Zat}), $i.e. \; \hat{Z}_{a,t+1} := \frac{\sum_{l\le t:A_l=a}Y_{al}}{N_{a,t+1}}$, according to (\ref{eq:alpha_q}) we propose

\begin{align}
    q_{a,t+1} = \frac{ (\hat{Z}_{a,t+1}  + \sigma_{a,t+1})^\frac{1}{\alpha}}{\sum_{b=1}^K( \hat{Z}_{b,t+1}  + \sigma_{b,t+1})^\frac{1}{\alpha}},
\end{align}
define $Z_a := Z\pi_a^{(\alpha)}$, this leads to the regret
\begin{align}
     R(q_t) = \frac{1}{\alpha(\alpha-1)}\left(\sum_a \hat{q}_{at}^{1-\alpha}\pi^{(\alpha)}_a
                    - (\sum_a (\pi^{(\alpha)}_a)^\frac{1}{\alpha})^\alpha\right)
        = \frac{1}{\alpha(\alpha-1)Z}\left(\sum_a \hat{q}_{at}^{1-\alpha}Z_a
                    - (\sum_a Z_a^\frac{1}{\alpha})^\alpha\right)
\end{align}

\begin{theorem}
\label{theorem_alpha}
The (instantaneous) regret for general $\alpha$-loss for $\alpha \in (0,2]$ is upper bounded by
$\frac{2^{\frac{1}{2\alpha}+2}c^{\frac{1}{2\alpha}+1}K^\frac{1}{2}}{\alpha^2Z\sqrt{t}}\left(\sum_a\tau_aZ_a^{\frac{1}{2\alpha}-1}\right)\sqrt{\sum_a\tau_a^\frac{1}{\alpha}}(\log t)^{\frac{1}{2}(\frac{1}{2\alpha}+1)}$.
\end{theorem}

Recall that with proposal 
\begin{align}
    q_{a,t+1} = \frac{ (\hat{Z}_{a,t+1}  + \sigma_{a,t+1})^\frac{1}{\alpha}}{\sum_{b=1}^K( \hat{Z}_{b,t+1}  + \sigma_{b,t+1})^\frac{1}{\alpha}}
\end{align}
where  $\sigma_{at} = c\tau_a\sqrt{\frac{\log t}{N_{at}}}, \hat{Z}_{a,t+1} := \frac{\sum_{l\le t:A_l=a}Y_{al}}{N_{a,t+1}}, N_{at}$ is the number of times arm $a$ has been picked up until time $t$, and $Y_{al} := \left(\frac{f(x_l)}{g_a(x_l)}\right)^\alpha$, we would like to bound the expectation of regret
$
    R(q_t) = \frac{1}{\alpha(\alpha-1)Z}\left(\sum_a \hat{q}_{at}^{1-\alpha}Z_a
                    - (\sum_a Z_a^\frac{1}{\alpha})^\alpha\right)
$,
where $Z_a = \int_{\mathcal{X}_a}\left(\frac{f(x)}{g_a(x)}\right)^\alpha g_a(x)dx$.

\begin{proof}
We follow the same proof structure as in Theorem \ref{theorem}. Recall
\begin{align}
    B_t:&=\{|\hat{Z}_{as}-Z_{as}| < \sigma_{as} \;\; \forall 1\leq a\leq K, \left[\frac{t}{2}\right] < s \leq t \}\\
     C_{t} :&= \{N_{at} > \beta_{at} \;\forall\; 1\le a\le K \}
\end{align}
and we re-define $\beta_{at}:= \frac{tZ_{a}^\frac{1}{\alpha}}{4(\sum_a (Z_a + 2c\tau_a\sqrt{\log t})^\frac{1}{\alpha}} $. The expected reward at time $t$ therefore can be written as:
\begin{align}
    \EE[R_t] 
    &=\EE[R_t|B_t\cap C_t]\Prob(B_t\cap C_t)
    +\EE[R_t|B_t\cap C_t^c]\Prob(B_t\cap C_t^c)
        + \EE[R_t|B_t^c]\Prob(B_t^c)
    \label{eq:alpha_decom}
\end{align}

For the first term in (\ref{eq:alpha_decom}), we have 
\begin{align}
\EE[R_t|B_t\cap C_t]&\Prob(C_t|B_t)\Prob(B_t)
 \leq  \EE[R_t|N_{at} \geq \beta_{at} \forall a, B_t] \\
&\leq \frac{1}{Z\alpha(\alpha-1)}\EE\left[\frac{\sum_a Z_a^\frac{1}{\alpha}}{\left(\sum_a(Z_a + 2\sigma_{at})^\frac{1}{\alpha}\right)^{1-\alpha}} - \left(\sum_a  Z_a^\frac{1}{\alpha}\right)^\alpha\right]\\
&= \frac{1}{Z\alpha(\alpha-1)}\EE\left[\left(\sum_a Z_a^\frac{1}{\alpha}\left(\sum_a(Z_a + 2\sigma_{at})^\frac{1}{\alpha}\right)^{\alpha-1} - \left(\sum_aZ_a^\frac{1}{\alpha}\right)^\alpha\right)\right]
\label{eq:tmp}\\
\end{align}

We apply Taylor's expansion twice to functions $h(x) = x^\frac{1}{\alpha}$ and $h(x) = x^{\alpha-1}$ to get
\begin{align}
   & \left(\sum_a(Z_a +  2\sigma_{at})^\frac{1}{\alpha}\right)^{\alpha-1} = 
    \left(\sum_a (Z_a^\frac{1}{\alpha}) + \frac{2}{\alpha}\sum_aZ_a^{\frac{1}{\alpha}-1}\sigma_{at} + \mathcal{O}(\sigma_{at}^2)\right)^{\alpha-1}\\
    &= \left(\sum_aZ_a^\frac{1}{\alpha}\right)^{\alpha-1} + \frac{2(\alpha-1)}{\alpha}\left(\sum_aZ_a^\frac{1}{\alpha}\right)^{\alpha-2}\left(\sum_aZ_a^{\frac{1}{\alpha}-1}\sigma_{at}\right) + \mathcal{O}(\sigma_{at}^2)
\end{align}
hence 
\begin{align}
    (\ref{eq:tmp}) &= \frac{1}{Z\alpha(\alpha-1)}\EE\left[\frac{2(\alpha-1)}{\alpha}\left(\sum_a Z_a^\frac{1}{\alpha}\right)^{\alpha-1}\left(\sum_aZ_a^{\frac{1}{\alpha}-1}\sigma_{at}\right) +  \mathcal{O}(\sigma_{at}^2)\right]\\
    &\leq \frac{2c}{Z\alpha}\left(\sum_aZ_a^\frac{1}{\alpha}\right)^{\alpha-1}\left(\sum_a\tau_aZ_a^{\frac{1}{\alpha}-1}\sqrt{\frac{\log t}{\beta_{at}}}\right)+  \mathcal{O}\left(\tau_a^2\frac{\log t}{\beta_{at}}\right) \\
    &=\frac{2c}{Z\alpha^2}\left(\sum_aZ_a^\frac{1}{\alpha}\right)^{\alpha-1}\sum_a\tau_aZ_a^{\frac{1}{\alpha}-1}\sqrt{\frac{4\log t\sum_a(Z_a+2c\tau_a\sqrt{\log t})^\frac{1}{\alpha}}{tZ_{a}^\frac{1}{\alpha}}} + \textit{lower order term}\\
    &\leq \frac{4c}{Z\alpha^2\sqrt{t}}\left(\sum_aZ_a^\frac{1}{\alpha}\right)^{\alpha-1}\sum_a\tau_aZ_a^{\frac{1}{2\alpha}-1}\sqrt{\log t\left(\sum_a Z_a^\frac{1}{\alpha} + \sum_a(2c\tau_a)^\frac{1}{\alpha}(\log t)^\frac{1}{2\alpha}\right)} + \textit{lower order term}\\
    &= \frac{2^{\frac{1}{2\alpha}+2}c^{\frac{1}{2\alpha}+1}K^\frac{1}{2}}{\alpha^2Z\sqrt{t}}\left(\sum_a\tau_aZ_a^{\frac{1}{2\alpha}-1}\right)\sqrt{\sum_a\tau_a^\frac{1}{\alpha}}(\log t)^{\frac{1}{2}(\frac{1}{2\alpha}+1)}+ \textit{lower order term}
\end{align}

We proceed by bounding the probability $\Prob(C_t^c|B_t)$. Conditioning on $B_t$ and for any $\frac{t}{2} < s < t$, we have

\begin{align}
\q_{as} &= \frac{\left(\hat{Z}_{as}+ \sigma_{as}\right)^\frac{1}{\alpha}}{\sum_b\left(\hat{Z}_{bs}+ \sigma_{bs}\right)^\frac{1}{\alpha}}
\geq \frac{Z_a^\frac{1}{\alpha}}{\sum_b (Z_b + 2\sigma_{bs})^\frac{1}{\alpha}}  \;\;\;\; \textit{ by definition of $B_t$} \\
 \implies& \sum_{s=1}^tq_{as} \geq \sum_{s=\frac{t}{2}+1}^t q_{as} \geq 
\frac{tZ_a^\frac{1}{\alpha}}{2\sum_a(Z_a + 2 c\tau_a\sqrt{\log t})^\frac{1}{\alpha}} =2\beta_{at}\;\;\;\;\textit{ since $s \leq t$ and $N_{as} \geq 1$}\\
\implies& \Prob\left(N_{at} < \beta_{at} | B_t\right) \leq 
    \Prob\left(N_{at} < \frac{1}{2}\sum_{s=1}^tq_{as} | B_t\right) \leq
    2\exp\left(-\frac{1}{2t}\beta_{at}^2\right) \;\;\;\textit{by lemma \ref{Hoeffding_MDS}}
\end{align}

Hence for the second term of  (\ref{eq:alpha_decom}), we have 

\begin{align}
&\EE[R_t|B_t\cap C_t^c]\Prob(C_t^c|B_t)\Prob(B_t) \\
&\leq \frac{2}{Z\alpha(\alpha-1)}\EE\left[\left(\frac{2(\alpha-1)}{\alpha}\left(\sum_a Z_a^\frac{1}{\alpha}\right)^{\alpha-1}\left(\sum_aZ_a^{\frac{1}{\alpha}-1}\sigma_{at}\right) +  \mathcal{O}(\sigma_{at}^2)\right)|B_t\cap C_t^c\right]\sum_a\exp\left(-\frac{1}{2t}\beta_{at}^2\right) \\
&\leq \frac{4}{Z\alpha^2}\left(\sum_a Z_a^\frac{1}{\alpha}\right)^{\alpha-1}\left(\sum_aZ_a^{\frac{1}{\alpha}-1}\EE[\sigma_{at}|B_t\cap C_t^c]\right)\sum_a\frac{2t}{\beta_{at}^2} + \textit{lower order term}\\
&\leq  \frac{4c\sqrt{\log t}\sum_a\tau_a}{Z\alpha^2}\left(\sum_a Z_a^\frac{1}{\alpha}\right)^{\alpha-1}\left(\sum_aZ_a^{\frac{1}{\alpha}-1}\right)\sum_a\frac{32\left(\sum_a(Z_a + 2c\tau_a\sqrt{\log t})^\frac{1}{\alpha}\right)^2}{tZ_a^\frac{2}{\alpha}} + \textit{lower order term}\\
&\leq  \frac{2^{7+\frac{2}{\alpha}}c^{1+\frac{2}{\alpha}}(\log t)^{\frac{1}{2} + \frac{1}{\alpha}}\sum_a\tau_a(\sum_a\tau_a^\frac{1}{\alpha})^2}{\alpha^2tZ}\left(\sum_a Z_a^\frac{1}{\alpha}\right)^{\alpha-1}\left(\sum_aZ_a^{\frac{1}{\alpha}-1}\right)\sum_aZ_a^{-\frac{2}{\alpha}}+ \textit{lower order term}
\end{align}

The proof of $\Prob(B_t^c)$ being small is the same as that in theorem \ref{theorem}. Recall
\begin{align}
\Prob(B_t^c) \leq \sum_{a=1}^K \sum_{s=\frac{t}{2}}^t \Prob(|\hat{Z}_{as} - Z_a| > \sigma_{as})
\leq \frac{2K}{t}
\end{align}

Therefore for the thris term in (\ref{eq:alpha_decom}), we have 
\begin{align}
&\EE[R_t|B_t^c]\Prob(B_t^c) 
\leq \frac{2K}{\alpha(\alpha-1)Zt}\EE\left[\left(\sum_a (\hat{Z}_{at} + \sigma_{at})^\frac{1-\alpha}{\alpha}Z_a  \left(\sum_a(\hat{Z}_{at} + \sigma_{at})^\frac{1}{\alpha}\right)^{\alpha-1} - \left(\sum_aZ_a^\frac{1}{\alpha}\right)^\alpha\right)|B_t^c\right] \\
&\leq \frac{2K}{\alpha(\alpha-1)Zt}\left(\sum_a (c\tau_a\sqrt{\log t})^\frac{1-\alpha}{\alpha}  Z_a  \EE\left[\left(\sum_a(\hat{Z}_{at} + \sigma_{at})^\frac{1}{\alpha}\right)^{\alpha-1}|B_t^c\right] - \left(\sum_aZ_a^\frac{1}{\alpha}\right)^\alpha\right)\\
&\leq \frac{2K}{\alpha(\alpha-1)Zt}\left(\sum_a (c\tau_a\sqrt{\log t})^\frac{1-\alpha}{\alpha}  Z_a  \left(\sum_a\EE[(\hat{Z}_{at} + \sigma_{at})^\frac{1}{\alpha}|B_t^c]\right)^{\alpha-1} - \left(\sum_aZ_a^\frac{1}{\alpha}\right)^\alpha\right) \;\; \textit{by Jensen}\\
&\leq \frac{2K}{\alpha(\alpha-1)Zt}\left(\sum_a (c\tau_a\sqrt{\log t})^\frac{1-\alpha}{\alpha}  Z_a  \left(\sum_a(
\mathcal{O}(t\sqrt{\log t}) + Z_a
)^\frac{1}{\alpha}\right)^{\alpha-1} - \left(\sum_aZ_a^\frac{1}{\alpha}\right)^\alpha\right)\\
&\leq \frac{2K}{\alpha(\alpha-1)Zt}\left(\sum_a (c\tau_a\sqrt{\log t})^\frac{1-\alpha}{\alpha}  Z_a  \left(\sum_a(Z_a^\frac{1}{\alpha} + \mathcal{O}((t\sqrt{\log t})^\frac{1}{\alpha})\right)^{\alpha-1} - \left(\sum_aZ_a^\frac{1}{\alpha}\right)^\alpha\right)\\
&\leq \frac{2K}{\alpha(\alpha-1)Zt}\left(\sum_a (c\tau_a\sqrt{\log t})^\frac{1-\alpha}{\alpha}  Z_a  \left(\left(\sum_aZ_a^\frac{1}{\alpha}\right)^{\alpha-1} + 
\mathcal{O}((t\sqrt{\log t})^\frac{\alpha-1}{\alpha})
\right) - \left(\sum_aZ_a^\frac{1}{\alpha}\right)^\alpha\right)\\
& = \frac{2Kc^\frac{1-\alpha}{\alpha}\sum_a\tau_a^\frac{1-\alpha}{\alpha}(\sum_a Z_a)}{\alpha(\alpha-1)Z} t^{-\frac{1}{\alpha}}  + \textit{lower order term}
\end{align}

Summing over $t$ gives the cumulative regret is dominated by the first term in (\ref{eq:alpha_decom}) which is $\mathcal{O}(t^{-\frac{1}{2}}(\log t)^{\frac{1}{2}(\frac{1}{2\alpha}+1)})$.

\end{proof}

\end{document}